\documentclass{article}

\usepackage[preprint,nonatbib]{neurips_2020}

\usepackage[utf8]{inputenc} 
\usepackage[T1]{fontenc}    
\usepackage{hyperref}       
\usepackage{url}            
\usepackage{booktabs}       
\usepackage{amsfonts}       
\usepackage{nicefrac}       
\usepackage{microtype}      

\usepackage{microtype}
\usepackage{graphicx}
\usepackage{subfigure}
\usepackage{booktabs} 

\usepackage{times}
\usepackage{helvet}
\usepackage{courier}
\usepackage{booktabs}       
\usepackage{graphicx} 
\usepackage{amsfonts}       
\usepackage{nicefrac}       
\usepackage{amsmath}
\usepackage{amstext}
\usepackage{algorithm}
\usepackage[noend]{algpseudocode}
\usepackage{amsthm}
\usepackage{dsfont}

\usepackage{xcolor}
\usepackage{tabularx}
\usepackage{booktabs}

\newtheorem{theorem}{Theorem}
\newtheorem{fact}{Fact}
\newtheorem{defn}{Definition}
\newtheorem{corr}{Corollary}
\newtheorem{prop}{Proposition}

\newtheorem{lemma}{Lemma}

\usepackage{mathtools}
\usepackage{comment}
\usepackage{bm}
\def\Pr{\mathop{\rm Pr}\nolimits}

\DeclareMathOperator*{\argmin}{arg\,min}
\DeclareMathOperator*{\argmax}{arg\,max}

\DeclareMathOperator{\Tr}{Tr}

\newcommand{\Dcal}{\mathcal{D}}

\renewcommand \vec [1]{\bm{#1}}

\title{Evaluating and Rewarding Teamwork Using Cooperative Game Abstractions}

\author{%
  Tom Yan\\
  Carnegie Mellon University\\
  Facebook AI Research \\
  \texttt{tyyan@cmu.edu} \\
  \And
  Christian Kroer\\
  Columbia University\\
  Facebook Core Data Science\\
  \texttt{christian.kroer@columbia.edu}\\
  \And
  Alexander Peysakhovich\\
  Facebook AI Research\\
  \texttt{alexpeys@fb.com}
}

\begin{document}

\maketitle

\begin{abstract}
Can we predict how well a team of individuals will perform together? How should individuals be rewarded for their  contributions to the team performance? Cooperative game theory gives us a powerful set of tools for answering these questions: the Characteristic Function (CF) and solution concepts like the Shapley Value (SV). There are two major difficulties in applying these techniques to real world problems: first, the CF is rarely given to us and needs to be learned from data. Second, the SV is combinatorial in nature. We introduce a parametric model called cooperative game abstractions (CGAs) for estimating CFs from data. CGAs are easy to learn, readily interpretable, and crucially allow linear-time computation of the SV. We provide identification results and sample complexity bounds for CGA models as well as error bounds in the estimation of the SV using CGAs. We apply our methods to study teams of artificial RL agents as well as real world teams from professional sports.
\end{abstract}

\section{Introduction}
Suppose we have a group of individuals out of which we need to select a team to perform a task. Besides maximizing team performance, we also wish to reward individuals fairly for their contributions to the team \cite{moulin2004fair}. This general problem arises in many real world contexts: choosing athletes for a sports team \cite{li2018learning}, choosing workers for a project \cite{salas2005there}, choosing a subset of classifiers to use in an ensemble \cite{rokach2010ensemble} etc. In this paper we ask: how can we use data on past performance to figure out which individuals complement each other? How can we then \emph{fairly} compensate team members? 

Standard game theory (sometimes called `non-cooperative' game theory) explicitly specifies actions, players, and utility functions \cite{osborne1994course}. By contrast, cooperative game theory abstracts away from the `rules of the game' and simply has as primitives the agents and the characteristic function (henceforth CF) \cite{brandenburger2007cooperative}. The CF measures how much utility a coalition can create. Solution concepts in cooperative game theory have been developed to be `fair' divisions of the total utility created by the coalition. These solution concepts can be viewed either as prescriptive (i.e. this is what an individual `deserves' to get given their contribution) or predictive of what will happen in real world negotiations, where the intuition is that coalitions (or individuals) that don't receive fair compensations will opt to leave the game and simply transact amongst themselves.

These tools are useful for answering our main questions. The CF tells us how well a team will perform and the solution concepts will tell us how to divide value across individuals. For the purposes of this paper, we consider one of the most prominent solution concepts: the Shapley Value (SV). However, there are two hurdles to overcome. 

\begin{enumerate}
    \item The CF is unknown to us, and is combinatorial in nature, thus requiring a sensible parametric model through which we can learn the CF from team performance data.
    \item The SV requires an exponential number of operations to compute.
\end{enumerate}

We introduce the cooperative game abstraction (CGA) model that simultaneously addresses \emph{both} of these issues. In addition, CGA models are interpretable so as to aid analysts in understanding group synergy. Our main idea is motivated by a particular decomposition of the CF into an additive series of weights that capture $m$-way interaction between the $n$ players for $m=1,...,n$. When we zero out terms of order order $k+1$ and higher, this leaves behind an abstraction, a sketched version of the real cooperative game, which we refer to as a $k$th order CGA. 


\textbf{Our Contribution:} To the best of our knowledge, we are the first to estimate characteristic functions with lossy \emph{abstractions} \cite{gilpin2007lossless} of the true characteristic function using parametric models, and bound the error of the estimated CF and SV. The second order variant of the CGA was first proposed in \cite{deng1994complexity}. We generalize this work to study CGA models of \emph{any order}. Our theoretical contributions are as follows: (i) sample complexity characterization of when a CGA model of order $k$ (for any order $k$) is identifiable from data (ii) sensitivity analysis of how the estimation error of the characteristic function propagates into the downstream task of estimating the Shapley Value.

Empirically, we first validate the usefulness of CGAs in artificial RL environments, in which we can verify the predictions of CGAs on counterfactual teams. Then, we model real world data from the NBA, for which we do not have ground truth, using CGAs and show that its predictions are consistent with expert knowledge and various metrics of team strength and player value.



\section{Related Work}

Past works on ML for cooperative games have largely been theoretical and focus either on estimating the CF or estimating the Shapley Value directly without the CF. This differs from our goal, which is to model \emph{both} with a provably good model that demonstrates sound performance on real world data.

\textbf{Modeling Characteristic Functions:} As mentioned prievously, \cite{deng1994complexity} is the first to consider what we consider the second order variant of CGA. However, its focus was on the computational complexity of the \emph{exact} computation of the Shapley Value. We consider the generalization of this representation to any order and are concerned with using lower rank CGA as an abstraction of complex games for \emph{computational tractability}. As the low rank CGA is a \emph{lossy} estimator of the true CF, we study and obtain theoretical bounds on the estimation errors of what we aim to compute: the CF and the SV.

A related work is \cite{feige2015unifying} which proposes the MGH model for CFs. While the MGH model is like CGA in that both are complete representations, it contains nonlinearity that makes it harder to optimize and \emph{interpret}. More crucially, the MPH model \emph{does not} admit an easy computation of the SV. On the other hand, there are succinct representation models proposed for CFs that do allow the SVs to be readily computed. These are algebraic decision diagrams \cite{bahar1997algebric} and MC-nets \cite{ieong2005marginal}, which represent CFs with a set of logical rules. However, the key drawback is that these models cannot be readily parameterized in tensor form and optimized using modern auto-grad toolkits, unlike the CGA.

Lastly, there has also been work in learning theory \cite{balcan2015learning} that examines conditions under which a characteristic function can be PAC learned from samples. This work is concerned only with the theoretical learnability of the CF (and not the SV) for \emph{certain classes} of cooperative games. By contrast, we study a concrete, parametric model that can approximate the CF of \emph{any cooperative game}, study how approximation noise propagates into the SV and empirically verify that the model obtains good performance on real data.

\textbf{Computing the Shapley Value:} There has also been work that directly approximates the Shapley Value, without first learning the CF \cite{balkanski2017statistical}. This differs from our goal in that we are interested in estimating \emph{both} the Shapley and the CF. The latter is needed for applications such as counterfactual team performance prediction and optimal team formation, as we will demonstrate in the experiments.

\textbf{Team Performance Analysis from Data:} 
We note that all of the work cited above are theoretical and do not test their model on real world data. \cite{li2018learning} is one empirical work that does. They model e-Sports team performances using a 2nd order CGA. Our work differs in that i) we generalize their model and study CGA models \textit{any order} to obtain comprehensive sample complexity bounds ii) we are interested in \emph{fair payoff assignment} in addition to team strength. To this end, we show that CGA allows for easy computation of SV and derive noise bounds for the estimated SV.

\textbf{Abstraction in Games:} Abstraction is an idea often used in game theory to make the computation of solution concepts such as the Nash Equilibrium (NE) tractable. One can efficiently solve for the NE of a abstracted game and lift the strategy to the original game. In non-cooperative game theory, the relationship between the quality of abstraction and the quality of the lifted strategy with respect to the original game has been heavily studied \cite{lanctot2012no,kroer2014extensive,kroer2016imperfect}. Our analysis characterizes the relationship between the abstractions and the solution concept, here being the Shapley Value, for any cooperative game. To the best of our knowledge, our work is the first to apply abstraction for computational tractability in the context of cooperative games.

\section{Cooperative Game Theory Preliminaries}
We begin with definitions in cooperative game theory.

\begin{defn}
A \textbf{cooperative game} is defined by:
\begin{enumerate}
\item A set of agents $A = \lbrace 1, \dots, n \rbrace$ with generic element $i$
\item A characteristic function $v: 2^A \to \mathbb{R}$
\end{enumerate}
\end{defn}

We will refer to a subset of agents $C \in 2^A$ for which $v(C)$  measures how much utility a team $C$ can create and divide amongst themselves. A `fair division' of this value can be given according to the Shapley Value.

\begin{defn}
The Shapley Value of an agent $i$ with respect to team $A$ is:
 $$\varphi_{i}(v) =  \sum_{S \subseteq {A \setminus i}}  \frac{|S|! (n - |S| - 1)!}{ n! } (v (S \cup \lbrace i \rbrace) - v(S)) $$
\end{defn}

The Shapley Value is typically justified axiomatically. It is the unique division of total value that satisfies axioms of efficiency (all gains are distributed), symmetry (individuals with equal marginal contribution to all coalition get the same division), linearity (if two games are combined, the new division is the sum of the games' divisions), null player (players with $0$ marginal contribution to any coalition receive $0$ value). The Shapley value has been widely applied in ML, in domains such as cost-division \cite{tan2002application,nguyen2018resource}, feature importance \cite{lundberg2017unified}, and data valuation \cite{jia2019towards} to name a few.


\section{Cooperative Game Abstractions}
\subsection{Motivation}

To model the characteristic function $v$, a natural set of abstractions can be derived from the fact that the characteristic function $v$ can be decomposed into a sum of interaction terms across subsets of agents. In what follows, we will denote abstractions of $v$ as $\hat{v}$.

\begin{fact}
  \label{thm:v interaction form}
  There exists a set of values $\omega_S$ for each $S = \lbrace i_1, \dots, i_k \rbrace  \subseteq A$
  such that any characteristic function can be decomposed into its interaction form where:
  \begin{align}
    \label{eq:v interaction form}
  v(C) = \sum_{k=1}^{|C|} \sum_{S \in 2^C_{k}} \omega_{S}.
  \end{align}
  where $2^C_{k}$ is the set of all coalitions of size $k$.
\end{fact}

Note that Fact \ref{thm:v interaction form} implies that CGA is a \emph{complete representation}: a CGA model of order $n$ can model \emph{any} set function. Its downside is that it has $2^n$ parameters to be learned from data. We may elect to truncate higher order terms and use an order $k$ CGA model $\hat{v}$ to model $v$ instead:

\begin{defn}
A $k$th CGA model is parameterized by weight vector $\omega$, which includes a weight $\omega_{C}$ for all coalitions $C$ with $| C | \leq k$. The corresponding $v(C)$ is defined as in equation \ref{eq:v interaction form}.
\end{defn}


A key property of CGA models is that the Shapley Value may be computed from a simple weighted sum of the CGA parameters.

\begin{fact}
The Shapley Value of an individual $i$ with respect to players $A$ may be expressed as: 

$$\varphi_{i}(v) = \sum_{T \subseteq A \setminus \{i\}} \frac{1}{|T|  + 1} \omega_{T \cup \{i\}}$$
\end{fact}

\subsection{Learning a CGA}
We learn the CGA model from samples of coalition values from $v$. Given hypothesis class be $\mathcal{H}$, we perform empirical risk minimization (ERM) with criterion: $\min_{\hat{v} \in \mathcal{H}} \sum_{(C, v(C)) \in \mathcal{D}_P} (\hat{v}(C) - v(C))^2$



An important question that immediately follows is: when can a CGA model be identified from data? We define an exact identification notion as below:

\begin{defn}
Suppose that we have a set of hypotheses $\mathcal{H}$ from which we will choose $\hat{v}$ via minimization of the criterion above. Suppose the dataset $\mathcal{D}_P$ is actually generated via a true $v^* \in \mathcal{H}$, we say that $\mathcal{D}_P$ identifies $v^*$ if $v^*$ is the unique minimizer of the criterion.
\end{defn}

Identification is an important question for three reasons. We are interested in the parameters of the model since they will be used to (i) predict the performance of unseen teams (ii) compute the Shapley value (iii) understand complementarity and substitutability between team members. If there are multiple sets of parameters consistent with the data, then none of the inferences we perform to answer those questions (e.g the marginal contribution of players) will be well defined.

We now give some sufficiency and necessity conditions on $\mathcal{D}_P$ for $v^*$ to be identified exactly. These results generalize known sample complexity bounds from 
~\cite{seshadri2019discovering}, expanding their bounds for only order $2$ to \emph{any} order $k$.

\begin{theorem}[Sufficiency for Identification]
\label{thm: suff_id}
Suppose $\mathcal{H}$ includes all $k^{th}$ order CGAs and $v^*$ is a $k^{th}$ order CGA. If $\mathcal{D}_P$ include performances from all teams of at least $k$ different sizes $s_1, \dots, s_k \in [k, n-1]$, then $\mathcal{D}_P$ identifies $v^*.$
\end{theorem}

\begin{theorem}[Necessity for Identification]
Suppose $\mathcal{H}$ includes all $k^{th}$ order CGAs and $v^*$ is a $k^{th}$ order CGA. If $\mathcal{D}_P$ contains performances of teams of only $m < k$ different sizes, then $\mathcal{D}_P$ does not always identify $v^*.$
\end{theorem}


These results show that if the order $k = O(1)$, identification is possible with $\text{poly}(n)$ samples and that if $k = O(n)$, the number of samples becomes exponential in $n$. Therefore, we suggest that practitioners should focus on the lowest order CGAs that they believe are suitable. For us, we find that low rank, second order CGAs demonstrate good performance in our experiments.

One consideration is that these bounds may be too pessimistic in requiring \emph{exact} recovery of the true $v$. In the appendix, we provide sample bounds for identification of CGA under a PAC/PMAC \cite{balcan2011learning} framework (Proposition 1). In particular, we have that under the looser, PMAC approximation notion, only $O(n)$ instead of $O(d_k)$ samples are needed for approximate estimation of most coalition values.

\section{Approximate Shapley Values}

With our approximation of the CF $\hat{v}$ in hand, we examine the fidelity of the SV computed from $\hat{v}$. We denote the approximated SV of player $i$ as $\varphi_{i}(\hat{v})$ and the real SV $\varphi_{i}(v)$. As is typical in sensitivity analysis, we derive bounds relating the error in $v$ to the error in the Shapley value.

These bounds may be of independent interest since often in ML applications $v$ is stochastic. For instance, SV is widely used in interpretability literature \cite{cohen2007feature, datta2015influence, lundberg2017unified,chen2018shapley,ghorbani2019data},
where $v$ is taken to be the model performance. The model performance is typically stochastic, since it is a function of the random samples of data used to train the model and the randomness in the optimization, which can converge to differing local optima due to the nonconvexity of the losses e.g of deep models. 

Let $\varphi(v)$ be the vector of Shapley Values. We start with a worst-case error bound for $\ell_2$ when the adversary can choose how to distribute a fixed amount of error into $v$ to construct $\hat{v}$.


\begin{theorem}
\label{thm: l2_worst_case}
The $\ell_2$ norm of the estimation error of the Shapley Values is bounded by:

 \begin{equation}
  \| \varphi(v) -  \varphi(\hat{v}) \|_2^2 \leq \frac{2}{n} \| v - \hat{v} \|_2^2
 \end{equation}
\end{theorem}

Though this result is \emph{tight}, it assumes a non-smooth, adversarial distribution that places infinite density on the eigenvector corresponding to the largest singular value of the SV operator. Below, we consider average case bounds assuming that the error is of fixed norm and drawn from a \textit{smooth} distribution; this type of assumption is often used in smooth analysis \cite{gupta2017pac}.

\begin{theorem}
\label{thm: l2_average_case}
Assuming that $v - \hat{v}$ is drawn from distribution $\mathcal{D}_{B_r}$ with support equal to a sphere and smooth in that $\kappa_0 \leq \Pr_{\mathcal{D}_{B_r}}(x) \leq \kappa_1$ for any point $x$ in its support, then:

 \begin{equation}
  \mathbb{E}_{v - \hat{v} \sim \mathcal{D}_{B_r}}[\| \varphi(v) -  \varphi(\hat{v}) \|_2^2] \leq \frac{6}{n} \frac{\kappa_1}{\kappa_0} \frac{\| v - \hat{v} \|_2^2}{2^n}
 \end{equation}

\end{theorem}

We can generalize these results to any noise distribution thus:

\begin{corr}
\label{corr: l2_average}
Suppose noise $v - \hat{v} \sim \mathcal{D}_n$ is such that its conditional distribution satisfies $\kappa_0(r) \leq \Pr_{\mathcal{D}_n}(x | \| x \|_2^2 = r^2) \leq \kappa_1(r)$ for all $r$ and $x$ in $\mathcal{D}_n$'s support, then:

$$\mathbb{E}_{v - \hat{v} \sim \mathcal{D}_n}[\| \varphi(v) -  \varphi(\hat{v}) \|_2^2] \leq \frac{6}{n} \mathbb{E}_r \left[ \frac{\kappa_1(r)}{\kappa_0(r)} \left ( \frac{r^2}{2^n} \right) \right]$$
\end{corr}

Intuitively, this means that if the error $v - \hat{v}$ is on average spread out in that $\mathcal{D}_n$ is fairly smooth in expectation across concentric spheres in its support, then the $\ell_2$ error of the Shapley value is small on average. Indeed, an astute reader may worry that only a $\frac{2}{n}$ reduction in the \textit{aggregate} approximation error $\| v - \hat{v} \|_2^2$ is not large enough since $v - \hat{v} \in \mathbb{R}^{2^n}$. Theorem~\ref{thm: l2_average_case} and Corollary~\ref{corr: l2_average} 
show that the SV actually induces a $\frac{6}{n} \frac{\kappa_1}{\kappa_0}$ scaling of the \textit{average} approximation error. 

We also obtain analogous worst and average-case $\ell_1$ bounds with scaling factors on the same order. Due to space constraints, please see Theorem 5 and 6 in the appendix for the results.











Lastly, we note that these bounds are general. In the appendix, we obtain a simple derivation of the CGA-specific bias, which can be plugged into these bounds for the SV bias. Note that bias in the estimation of the CF only arises due to model misspecification, i.e if order $k$ is used to model a game of order $r$ for $r > k$. This description covers \emph{all cases} as any CF of a game necessarily corresponds to a CGA model of a certain order (Fact \ref{thm:v interaction form}) and estimation error only arises due to a smaller order being specified. Certainly, we note that more refined bounds are a natural future extension to this work.

\section{Experiments}

\subsection{Virtual Teams}
We generate team performance data from the OpenAI particle environment \cite{lowe2017multi} \footnote{\url{https://github.com/openai/multiagent-particle-envs}}. 
The task in this environment is team-based and requires cooperation:
$3$ agents are placed in a map and $3$ landmarks are marked, and agents have a limited amount of time to reach the landmarks and are scored according to the minimum distance of any agent to any landmark. In addition, negative rewards are incurred for colliding with other agents. Thus, a team which can cooperate well is able to assign a single landmark per agent in real time and spread out to cover them without colliding with each other.


\begin{figure}[ht]

\begin{minipage}{.5\columnwidth}
  \centering
  \includegraphics[width=\columnwidth]{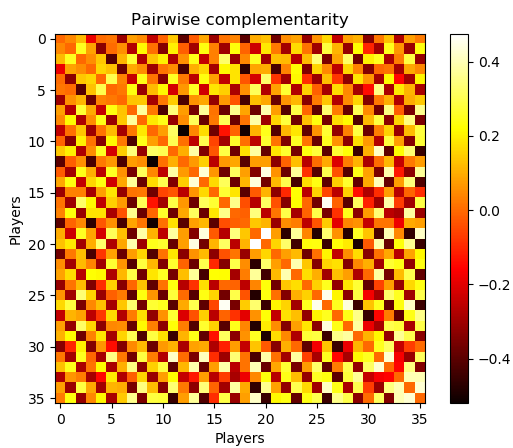}
\end{minipage}
\begin{minipage}{.5\columnwidth}
  \centering
  \includegraphics[width=\columnwidth]{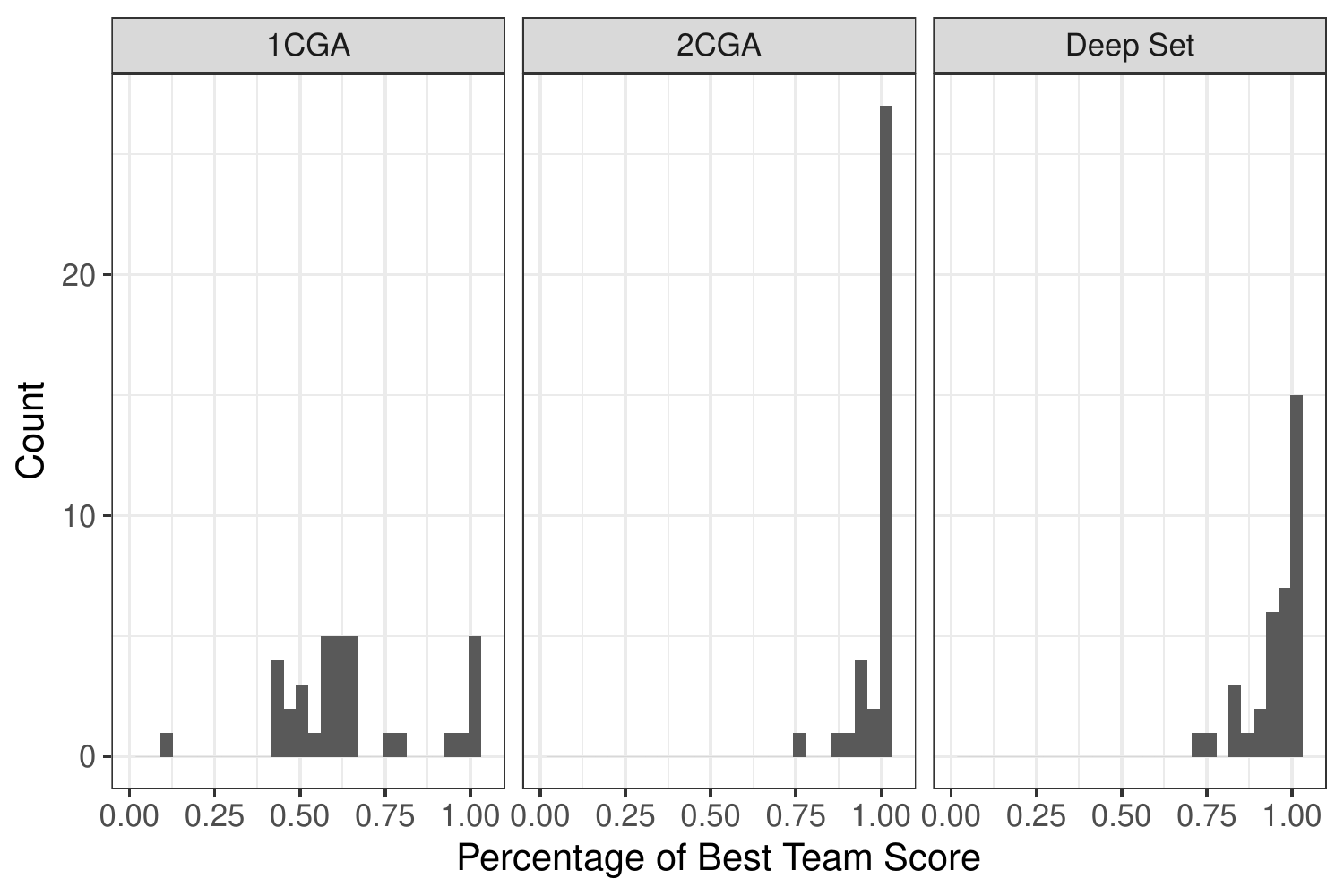}
\end{minipage}

\caption{Left: Interaction matrix from second order CGA model. Players are clustered by original training team ($\lbrace 0,1,2 \rbrace$ trained together as did $\lbrace 3,4,5 \rbrace$, etc...). We see complex patterns of complementarity and substitutability as well as a clear replication of the well known fact that agents that train together can coordinate much better than agents which are trained separately - this can be seen in the figure by the strong complementarity in the diagonal blocks of size $3$ compared to other $3 \times 3$ off-diagonal blocks.  Right: Histograms of ratios of the score attained by the completed team chosen by the model normalized by the score of the actual best team containing a given initial agent.}
\label{coop_nav}
\end{figure}

We train $12$ teams of agents ($36$ agents total) using the default algorithm and parameters from the OpenAI GitHub repo. We then evaluate all $\binom{36}{3}$ $= 7140$ mixed teams of these agents evaluated over 100000 episodes. The train/validation/test split is 50/10/40. We fit a baseline first order (where team = sum of members) CGA and a second order CGA to predicting the final score of each team. We also compare to a more general, state of the art model for learning set functions, DeepSet~\cite{zaheer2017deep}, which is designed purely for prediction. 


Since DeepSet contains more parameters than the CGA model, we expect it to fit data better. However, unlike the CGA model, Deepset is (i) less easily identified due to the larger sample complexity needed (ii) \textit{not readily interpretable} due to the non-linearity of $\phi$ (iii) and importantly, one cannot readily compute or estimate the Shapely. To compute the Shapley values exactly, one would have to first compute $v(C)$ for each coalition $C$, thus requiring $2^n$ feed-foward passes through the network. Even to approximate the Shapley value, it is known that $O(n\log n)$ evaluations of the model (network) are needed \cite{jia2019towards}. In contrast, to compute the Shapley with CGA, only one weighted sum of the CGA model parameters is needed and thus takes $O(1)$ number of evaluation.

\textbf{Prediction}: The first order CGA model achieves an test set MSE of of $.79$, the second order model achieves an order of magnitude smaller at $.07$. These results show that in this environment teams are not just sums of their parts. The DeepSet model achieves an MSE of $.042$, showing that we give up some predictive accuracy (but not that much) from using the simpler $2^{nd}$ order CGA. We emphasize that the goal of this experiment is \emph{not} to find the most predictive model. Rather, it is to show that the much smaller, second order CGA model is roughly comparable to Deepset, all the while conferring the advantages of: 1. being interpretable 2. allowing easy computation of the Shapley Value.

\textbf{Interpretability:} To the first point, we visualize the learned matrix $\widehat{V}$ of the second order CGA in a heatmap (Figure \ref{coop_nav}) that allows us to discern players that complement/substitute each other.

\textbf{Best Team Formation:} For each of the 36 agents we have trained, we ask: what is the best set of $2$ agents to add to them to make a team? More generally, this problem of optimal player addition is one often faced by real world sports teams, as they choose new players to draft or sign so as to further bolster their team performance. In this virtual setting, we can evaluate all possible additions to the team so as to gauge the predictive performance of our models.

In our setup, we restrict only to possible teammates which the original agent was not trained with. Figure \ref{coop_nav} shows the histogram of ratios of the score attained by the completed team, which was selected by the model, normalized by the score of the actual best team. While the first-order CGA fails to construct good teams (since it does not consider any complementarities), the second order CGA and DeepSet model achieve more than $\sim 95 \%$ of the possible value. Thus, the complementarity patterns learned via the 2nd order CGA are, in fact, important for this task. We also note that the second order CGA model outperforms DeepSet on this task.

\subsection{Real World Sports Teams}
We now consider a more complex, real world problem: predicting team performance in the NBA. We collect the last $6$ seasons of NBA games (a total of 7380 games) from Kaggle along with  the publicly available box scores
\footnote{ \url{https://www.kaggle.com/drgilermo/nba-players-stats}}. Unlike in the dataset above, we do not observe absolute team performance, rather we only observe relative performance (who wins). We model matchup outcomes using the Bradley–Terry model. In particular, given the team strengths, the probability of team $i$ winning in a match against team $j$ as: 

$$\Pr[w = 1 \mid \hat{v}, C_i, C_j] = \dfrac{\text{exp} (\hat{v} (C_i))}{\text{exp}(\hat{v} (C_i)) + \text{exp}(\hat{v} (C_j))}$$

This gives us a well defined negative log likelihood (NLL) criterion of the data $\mathcal{D}$, which we optimize with respect to $\hat{v}$. 
We set each team in each game to be represented by its starting lineup ($5$ individuals). Then we learn $\hat{v}$ such that it minimizes the negative log likelihood using standard batch SGD with learning rate $0.001$. Because basketball teams are of a fixed size (only one set of sizes), we use L2 regularization to choose one among the many possible set of models parameters.

As with the RL experiment above we compare a first order CGA, a second order CGA, and a DeepSet model. We split the dataset randomly into 80 percent training, 10 percent validation, and 10 percent test subsets. We set hyperparameters by optimizing the loss on the validation set. 

\subsubsection{Results}
\textbf{Prediction:} How well does the CGA perform in this task? We begin by studying an imperfect metric: out-of-sample predictive performance. First, we see that the NBA performance can be fit fairly well with just a first order CGA - that is, we can think of most teams roughly as the sum of their parts. The first order CGA yields an out of sample mean negative log likelhood of $-.631$ which is slightly improved to $-.627$ under the second order CGA. We do also experiment with a third order CGA which did not improve over the second order CGA performance. This suggests that the second order is an apt choice for the abstraction. Finally, we observe that the DeepSet model is not able to outperform the CGA yielding an out of sample mean NLL of $-.63.$

Overall, we find that predictive accuracy is low, at only about $\sim 65 \%$, as a result of the league being very competitive and teams being fairly evenly matched. Thus, predictive accuracy does not tell the whole story and is not the focus of the experiment. Note that the data at hand is observational and while players do move across teams and starting lineups change due to factors such as injuries, time in the season, etc... who plays with whom is highly correlated across years and starting lineups are endogenous (for example, a coach may not start one of their best players when playing a much weaker team to avoid risking injury). Thus, we cannot evaluate counterfactual teams. Instead, we supplement our the predictive analysis with analyses of the competing models to see if they are truly able to extract insights from the data consistent with NBA analytics experts.

\textbf{Unseen teams:} We consider teams the model has not seen: NBA All Star teams. During each season, fans and professional analysts vote to select `superstar' teams of players that then play each other in an  exhibition game, which is not included in our training data. We collect every All Star team from the time period spanning our training set and compare our second-order CGA model scores given to All Star teams with those of $1000$ randomly generated, `average' teams.

Recall that in the matchup datasets, the difference in scores between two teams is reflective of the probability that one team will win in a matchup. Thus, there is no natural zero point like when we are predicting $v$ directly and we have chosen one particular normalization where the average score is zero. If our model does generalize well, it should predict that these all star teams are far above average despite never seeing this combination of players in the training set.

We also investigate whether the CGA has learned things about whole teams (e.g ``the Cavaliers usually win'') and whether there is sufficient variation in starting lineups that we have learned the disentangled contributions of individual players (e.g the team's success is largely due to Lebron's brilliance). We investigate this by constructing synthetic `same-team-All-Star' teams where we replace each player in a real All Star team with a randomly selected teammate from their real NBA team from that year.  

Figure \ref{nba_dist} shows the distribution of scores for randomly constructed teams with red lines representing predicted scores for the real All Star teams and blue lines for predicted scores for the `same-team-All-Star' teams. These results show that the predictive performance of the CGA in win rate prediction comes from meaningful player-level assessment, not just that certain teams usually win (or lose).

\begin{figure}[ht]

\begin{minipage}{.5\columnwidth}
  \centering
  \includegraphics[width=\columnwidth]{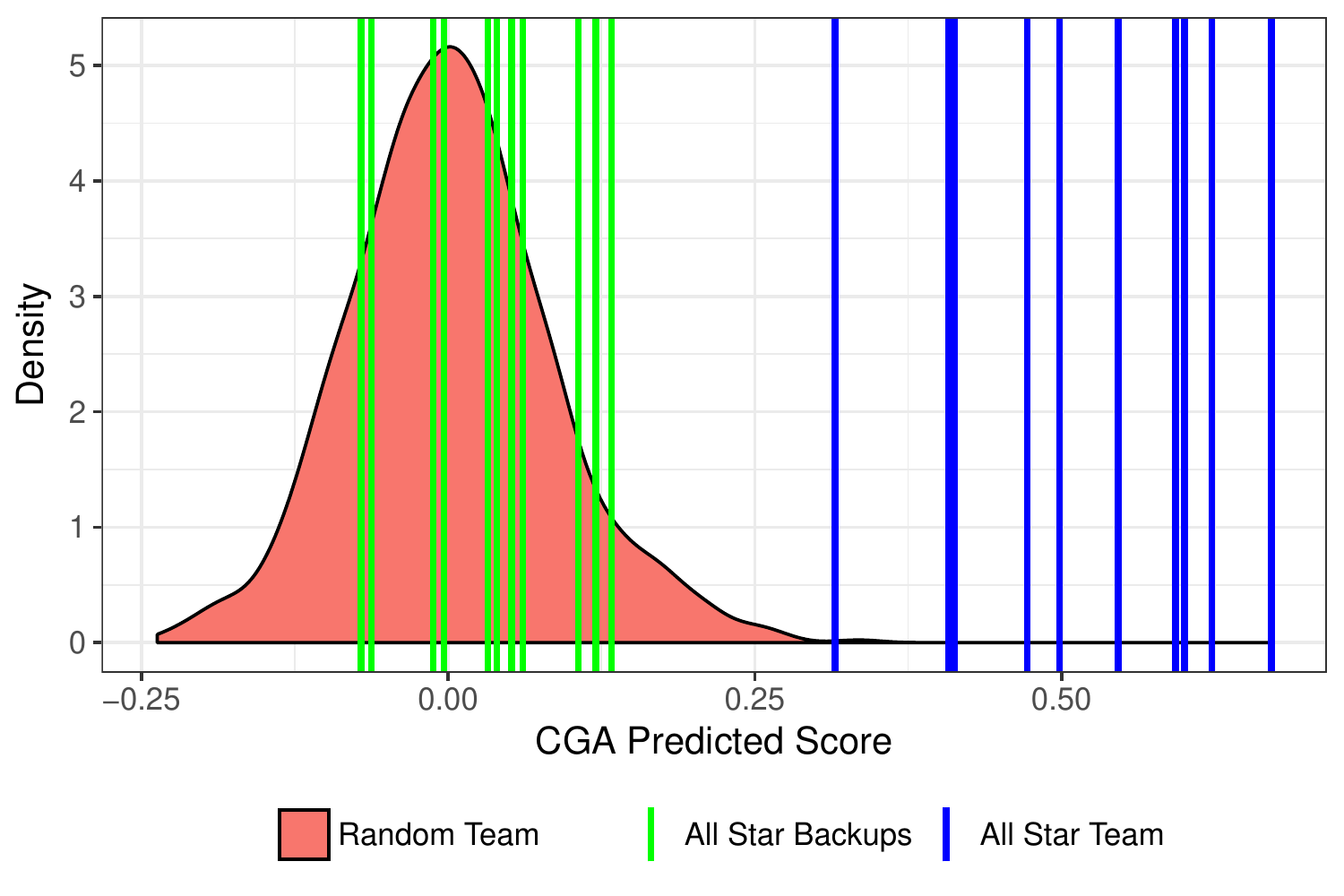}
\end{minipage}
\begin{minipage}{.5\columnwidth}
  \centering
  \includegraphics[width=\columnwidth]{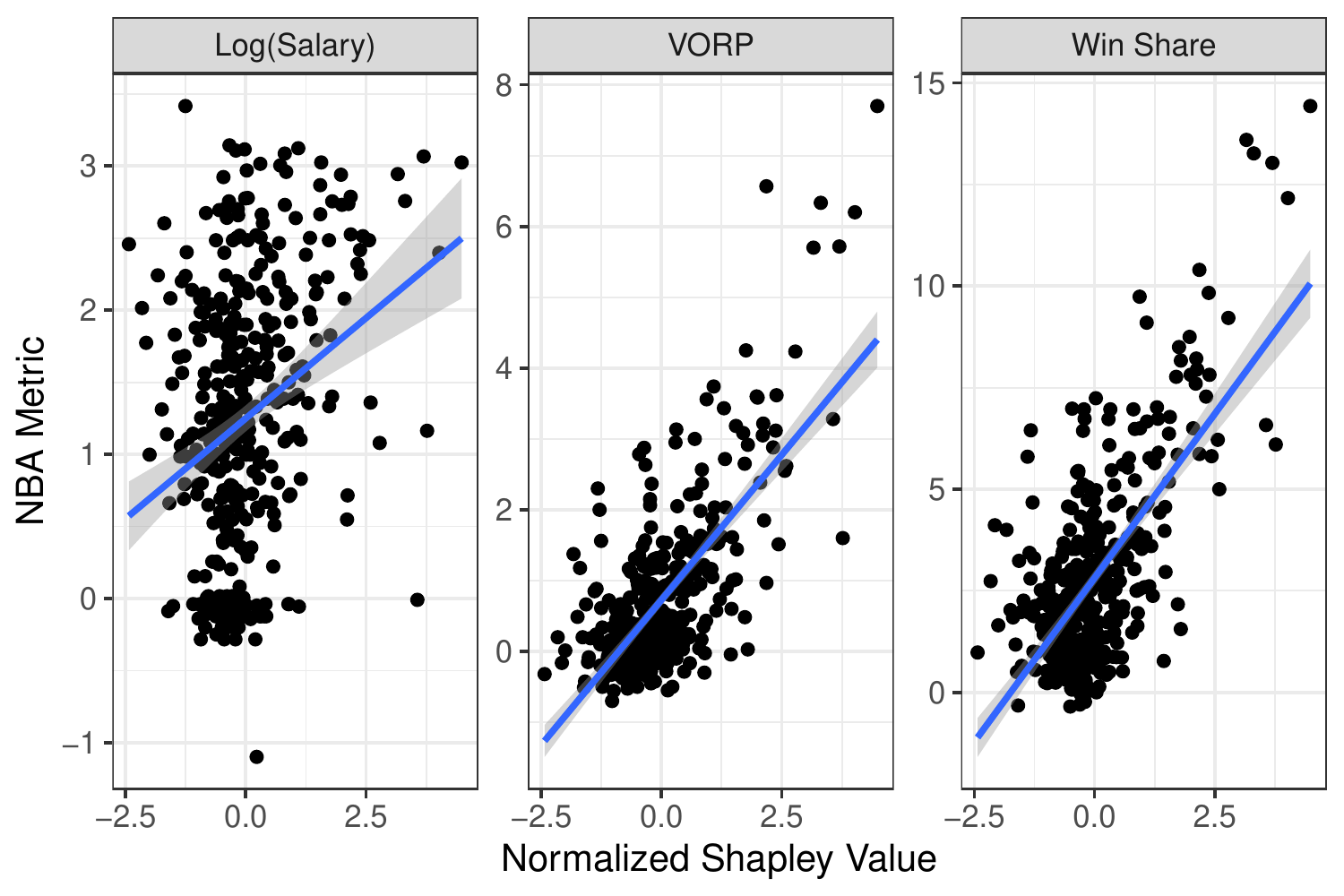}
\end{minipage}

\caption{Left Panel: The second order CGA predicts that All Star Teams are far above the $99^{th}$ percentile of random teams. Replacing each All Star with their team-level replacement gives much worse teams. These results show that the predictive performance of the CGA in win rate prediction comes from player-level assessment, and not just memorization of certain teams usually wining or losing. Right Panel: Marginal contributions of individual NBA players, as measured by the Shapley Value from the second order CGA, correlate well with measures of player-level value add used by NBA analysts (VORP, Win-Share) as well as market-level value-add (salary).}
\label{nba_dist}
\end{figure}

\textbf{Shapley Value as Individual Measure:} So far we have asked whether our CGA captures team-level performance. We now turn to asking whether it captures individual-level marginal contribution. For each team, we compute the team members' Shapley Values \textit{with respect to that team}. Since our dataset contains multiple years and individuals move across teams, we average an individual's computed Shapley values across all his teams. We correlate the Shapley Value based contribution scores with real world metrics used to evaluate basketball players' marginal contributions. We consider $3$ measures commonly used in NBA analytics.

First, we look at the value-over-replacement metric player\footnote{\url{https://www.basketball-reference.com/leaders/vorp_career.html}} (VORP). In basketball analytics, VORP tries to compute what would happen if the player were to be removed from the team and replaced by a random player in their position. Second, we look at win-share\footnote{\url{https://www.basketball-reference.com/about/ws.html}} (WS). Win-share tries to associate what percent of a team's performance can be attributed to a particular player. Finally, we use individual salaries, which are market measures of individual value add. Of course, a players' salary reflects much more than an individuals' contribution to team wins and losses (e.g their popularity, scarcity, etc...) and is extremely right tailed in the case of the NBA, so we consider its log. For each of these metrics, for each player, we average their values across the same years as our dataset.

Figure \ref{nba_dist} plots CGA Shapley values against these measures. We see that there is a strong positive relationship between the CGA predicted Shapley value and other measures of individual contribution. Taken together, these results suggest that CGA indeed learns meaningful individual-level contribution measures, in a way that is consistent with expert knowledge.

\textbf{Remark:} overall, our experiments highlight the computational benefit of CGAs. In many cases like the NBA, team sizes are small relative to the number of players. We show that this structural prior can be encoded in a low rank CGA model, which does just as well (or better) than more complex, agnostic estimators like DeepSet (our main baseline), and is also interpretable to the benefit of users.

\section{Conclusion}
Cooperative game theory is a powerful set of tools. However, the CF is combinatorial and computing solution concepts like the Shapley is difficult. We introduce CGAs as a scalable, interpretable model for approximating the CF, and easily computing the SV. We provide a bevy of theoretical and empirical results so as to guide the application of CGA to model real world data.

Non-cooperative Game Theory has received much attention from the Machine Learning and AI community \cite{shoham2007if,lazaridou2016multi,lanctot2017unified,lerer2017maintaining,brown2018superhuman}, while Cooperative Game Theory has been less explored. We believe that the intersection of Machine Learning and Cooperative Game Theory is rich with topics ranging from Multi-agent RL to Federated Learning. Our broader hope is that our work provides a springboard for future research in this area.

\section{Broader Impact}

In this work, we introduce and analyze a general model for team strength and player value. While our end goal is to ensure accurate assessment of team strength, and as a result \emph{fair} distribution of team value, there is the risk of model misspecification and resultant bias in the estimators. Furthermore, often times the team performance data we see is observational and the data we observe may be biased due to individuals being of disparate background. Indeed, accounting for such confounding factors is an important extension to our work that we would like to highlight.

\bibliographystyle{plain}
\bibliography{refs}

\onecolumn

\appendix

\section{Appendix to ``Evaluating and Rewarding Teamwork Using Cooperative Game Abstractions''}

\setcounter{fact}{0}
\setcounter{theorem}{0}
\setcounter{prop}{0}
\setcounter{corr}{0}

\subsection{Identification Theorem Proofs}
\label{app: id_theorems}

\begin{theorem}[Sufficiency for Identification]
Suppose $\mathcal{H}$ includes all $k^{th}$ order CGAs and $v^*$ is a $k^{th}$ order CGA. If $\mathcal{D}_P$ include performances from all teams of at least $k$ different subset sizes $s_1, \dots, s_k \in [k, n-1]$, then $\mathcal{D}_P$ identifies $v^*.$
\end{theorem}

\begin{proof}

Let $\vec{w}$ be the first-through-$k$'th order weights we seek to learn, with the first $n$ indices corresponding to $\omega_{S}$ such that $|S|=1$, the next $n \choose 2$ indices corresponding to $|S|=2$, and so on up through the last $\binom{n}{k}$ terms corresponding to $|S|=k$. Let $\vec{v}$ be the corresponding coalitional values we observe.

Finding a $k$'th-order CGA corresponding to $\mathcal{D}$ can be formulated as finding a solution to $M\vec{w} = \vec{v}$, where matrix $M$ is a matrix whose rows correspond to the data points and each entry in the matrix $\in \{0,1\}$. For a given datapoint $(S,v(S))$, the corresponding row has ones in all entries corresponding to interaction terms $\omega_T$ such that $T\subseteq S$. Note that we only consider subset sizes $\geq k$, since subsets sizes smaller than $k$ would not exhibit $k$th order interaction.

To show identifiability, it suffices to show that $M$ has rank equal to the column size, since otherwise the null space is non-empty and there exist multiple $\vec{w}$ which satisfies the equation. Equivalently, a full rank matrix ensures that the optimization criterion is strictly convex and that the minimizer is unique. Define matrix $M_{ntk}$ to be the submatrix consisting of all rows from all subsets of size $t$ and columns corresponding only to that of the $k$'th order weights.

\begin{center}
$M$ = \bordermatrix{~ & \text{first order weight} & ... & \text{$k$'th order weight} \cr
                  \text{rows from subsets of size $s_1$} & M_{ns_11} & ... & M_{ns_1k} \cr
                  ... & ... & ... & ... \cr
                  \text{rows from subsets of size $s_k$} & M_{ns_k1} & ... & M_{ns_kk} \cr}
\end{center}

We will now show that we can perform row reductions on the decomposition into submatrices, such that we end up with all zeroes below the antidiagonal.

First we note that every row in $M_{nbk}$ is a linear combination of rows in $M_{nak}$ for any $a < b$.
Consider a row $s_b$ corresponding to subset $\{i_1,...,i_b\}$. 
We take all the rows in $M_{nak}$ corresponding to subsets $s'$ where $s' \subseteq \{i_1,...,i_b\}$ and $|s'| = a$.
We sum all $\binom{b}{a}$ of these rows, and denote this row $s'_b$.
Looking at a particular $k$th order weight, say WLOG corresponding to $\{i_1,...,i_k\} \subseteq \{i_1,...,i_b\}$, there is a $1$ in this column in $s_b$.
This subset of size $k$ shows up in $\binom{b-k}{a-k}$ subsets of size $a$. Therefore, the corresponding entry in row $s'_b$ is $\binom{b-k}{a-k}$.
And so, we can derive that $\binom{b-k}{a-k}^{-1} s'_b = s_b$ as they both have the same support: every subset of size $k$ in $\{i_1,...,i_b\}$ can be found in a subset of $\{i_1,...,i_b\}$ of size $a$.

Thus we may use an appropriate multiple $\alpha$ of the first row to replace $M_{ns_ik}$ with a zero for any $i>1$. However, this changes the whole row, and so the $j$'th order term changes to $M_{ns_ij} - \alpha M_{ns1j}$.
But by the same logic as for $k$, summing the $l$'th weights gives the same row scaled by $\binom{b - l}{a - l}$, and thus $M_{ns_ij} - \alpha M_{ns1j} = \left(1 - \frac{\binom{b - l}{a - l}}{\binom{b-k}{a-k}}\right)M_{ns_ij}$.

The above shows that we can perform row reduction using the first row of submatrices in order to put zeroes in the last column while retaining all submatrices in other columns (up to rescaling). But now we may apply this logic inductively, by considering only the submatrices corresponding to first through $k-1$'th order weights and rows from subsets of size $s_2$ or greater, and so on. We get that the matrix $M$ looks as follows after row reduction:

\begin{center}
\bordermatrix{~ & \text{first order weight} & ... & \text{(k-1)th order weight} & \text{kth order weight} \cr
                  \text{rows from subsets of size $s_1$} & M_{ns_11} & ... & M_{ns_1(k-1)} & M_{ns_1k} \cr
                  \text{rows from subsets of size $s_2$} & M'_{ns_21} & ... & M_{ns_2(k-1)} & 0 \cr
                  ... & ... & ... & 0 & 0 \cr
                  \text{rows from subsets of size $s_k$} & M_{ns_k1} & ...& 0 & 0 \cr}    
\end{center}

where $M'_{ns_21}$ denotes submatrices above the antidiagonal that have been rescaled (note that the first row does not need rescaling). It is then sufficient to show that $M_{ns_1k}, M_{ns_2(k-1)} ..., M_{ns_k1}$ (note that each of these submatrices has more rows than columns) are all full rank to show $M$ is full rank.

To do this, we first prove a lemma.

\begin{lemma}
  \label{lem:identification_sublemma}
When $t \geq k$ and $n = t + k$, the matrix $M_{ntk}$ is full rank.
\end{lemma}
\begin{proof}
  
We will proceed by induction on $k$.

\emph{Base case ($k=1$)}: 
Since $k=1$ each row corresponds to an all-one row with a single zero for the agent left out. Since we have such a row for each agent that can be left out, we get $n$ linearly-independent rows.

Thus the matrix is full rank.

\emph{Induction step ($k>1$)}: 

Assuming this statement holds for orders $1,\ldots, k - 1$. We will prove the statement for when the order is $k$.
To do this we will use induction on $n$:

\emph{Base case ($n=2k$)}: 
$n = 2k \Rightarrow t = k$, and so $M_{ntk}$ is the identity matrix and is thus full rank.

\emph{Induction step ($n > 2k$)}: 
Assume the matrix is full rank for when number of players is equal to $2k, ..., n$.
To prove the matrix is full rank for $n + 1$, consider the following decomposition of $M_{(n+1)(n+1-k)k}$.

\begin{center}
\bordermatrix{~ & \text{weights of subsets including 1} & \text{weights of excluding including 1} \cr
                  \text{subsets including 1} & A & B \cr
                  \text{subsets excluding 1} & 0 & C \cr}    
\end{center}

Observe that matrix $B$ corresponds to $M_{n(n-k)k}$ and is thus full rank by induction hypothesis. This means we can use linear combinations of rows of $B$ to reduce rows in $C$.
In particular, we can performs row reductions such that we replace $C$ with zeroes:
For each row in $C$ corresponding to a team of size of $n + 1 - k$ selected from $[2,..., n]$, we consider all $n - k$ subsets of this team in $B$ and sum them. For any subset of this team of size $k$, then we see that it shows up in the sum: $\binom{ n + 1 - k -(k)}{n - k - (k)} = n + 1 -2k$ times. Therefore, the sum is $n + 1- 2k$ times the row in $C$.

Moreover, let $D$ be the $n + 1 - k$ rows from $A$ summed together when performing the row reduction, let the resultant matrix be $D$. The reduced matrix looks like:

\begin{center}
$M_{(n+1)(n+1-k)k}$ = \bordermatrix{~ & \text{weights of subsets including 1} & \text{weights of excluding including 1} \cr
                  \text{subsets including 1} & A & B \cr
                  \text{subsets excluding 1} & D & 0 \cr}    
\end{center}

Then, we observe that $D$ corresponds to a scaled version of $M_{n(n+1-k)(k-1)}$, which is full rank by the inductive assumption. 
The scaling factor is calculated as follows:
For a subset $\{1,...,k\}$, the $k - 1$ elements show up in the $n+1-k$ subset row of $C$, then shows up in $\binom{n + 1 -k -(k-1)}{n - k - (k-1)} = n + 2 - 2k$ of the $n-k$ subsets. And so, $D$ is a $-\frac{n + 2 -2k}{n + 1- 2k}$ scaled version of $M_{n(n+1-k)(k-1)}$.

$B$ remains unchanged after the row reduction and is full rank and thus the whole matrix is full rank.
\end{proof}

With this lemma in hand we can return to the main proof.
To complete the proof, we will show that for any $k$, any $n \geq 2k$ and any $t \in[k, n-k]$, $M_{ntk}$ is full rank. 

Note that $t \in[k, n-k] \Rightarrow \binom{n}{t} \geq \binom{n}{k}$ (as otherwise number of rows is already fewer than number of columns and the matrix will have rank less than column size).

We will use induction on $k$.

\emph{Base case ($k=1$)}: by Lemma~\ref{lem:identification_sublemma}, the matrix is full rank when $k=1$ for any team size $t$ and number of players $n$.

\emph{Induction step ($k>1$)}: assumes this holds for orders $1,...,k-1$ and any $t$ and $n$.
For order $k$, fix some $t \geq k$, we will show the matrix is full rank for all $n \geq t + k$ by induction on $n$.
For the base case $n = t + k$ the matrix is full rank by Lemma~\ref{lem:identification_sublemma}.
For the induction step assume $n>t+k$: assume the matrix is full rank when the number of players is in $\{t+k,...,n-1\}$.
Now when the number of players is $n$, we may decompose the matrix into columns corresponding to weights of $k$-size subsets containing player $1$, and rows into teams including or excluding player $1$.

\begin{center}
$M_{ntk} $= \bordermatrix{~ & \text{weights of subsets including 1} & \text{weights of subsets excluding 1} \cr
                  \text{subsets including 1} & U_1 & U_3 \cr
                  \text{subsets excluding 1} & 0 & U_2 \cr}    
\end{center}

In doing so, we first observe that
$U_1$ is exactly $M_{(n-1)(t-1)(k-1)}$ and is full rank from the induction hypothesis on $k$.
Secondly,
$U_2 = M_{(n-1)tk}$ and is thus full  rank by the induction hypothesis on $n$.
Therefore the matrix $M_{ntk}$ is full rank which concludes the inductive step on $n$. But this also concludes the inductive step on $k>1$, and thus we we get that all $M_{ntk}$ along the antidiagonal of $M$ are full rank. It follows that $M$ is full rank, and thus $M\vec{w}=\vec{v}$ has a unique solution.

Finally, because we are choosing $k$ subset sizes from $[k, n-1]$, it's easy to see that if we sort subset size $s$  by $\binom{n}{s}$, then the $j$th subset size in this sorted order $s_i$ is such that $\binom{n}{s_j} \geq \binom{n}{k - j + 1}$, which means the above condition applies.
\end{proof}

\begin{theorem}[Necessity for Identification]
Suppose $\mathcal{H}$ includes all $k^{th}$ order CGAs and $v^*$ is a $k^{th}$ order CGA. If $\mathcal{D}_P$ contains performances of teams of only $m < k$ different sizes, then $\mathcal{D}_P$ does not always identify $v^*.$
\begin{proof}

We will provide an instance when $k=2$ such that $v^*$ is not identified.
In that case $m=1$, and we may pick teams of size $n - 1$. That gives us $\binom{n}{n - 1} = n$ rows which is fewer than the number of columns $\binom{n}{2} + \binom{n}{1}$. Thus there will be more than one solution.

Moreover, the conditions specified in Theorem 1 are also tight in the sense that: if we allowed $m=k$ subset sizes, but over a wider interval, then $\mathcal{D}$ does not always identify $v^*.$ To see this, consider $k = 2$ again. Widening the interval means the inclusion of either subset size $k-1$ or $n$.

If we can pick $k-1$, consider $m=2$ subset sizes $k-1$ and $n-1$, which together gives $\binom{n}{1} + \binom{n}{n-1}$ rows, which is fewer than the number of columns $\binom{n}{2} + \binom{n}{1}$.

If we can pick $n$, consider $m=2$ subset sizes $n-1$  and $n$, which together gives $\binom{n}{n-1} + \binom{n}{n}$ rows, which is fewer than the number of columns $\binom{n}{2} + \binom{n}{1}$.

\end{proof}
\end{theorem}

\subsection{PAC Analysis}
\label{app: pac}

Another natural paradigm through which we may analyze sample complexity of learning a CGA is the PAC framework. Before we proceed, a word about why PAC bounds are not our main focus for sample complexity. One drawback of PAC bounds we considered is that it is only with high probability that \emph{most} coalition values are well approximated. Therefore, it could still be that there is one $\hat{v}(S)$ that is arbitrarily off. Thus, the resultant estimated Shapley value will inherit this large bias. Since we hope to use the estimated Shapley Value for fair credit assignment in practice, we opt for what may be considered more “pessimistic”, exact identification guarantees similar to those in ~\cite{seshadri2019discovering}. 

Below, we provide two results based on PAC and PMAC notions of approximation. We prove the result assuming that we have correct CGA order specification. The result follows similarly when a higher order than that of the true CF is specified.

Consider a random sample $S$ of $m$ $(C, v(C))$ data points with $C$ uniformly sampled from $2^A$. There are at most $m$ distinct coalitional values in that sample. Call them $\vec{v_{\widehat{S}}}$. We will solve $M^{nk}_{\widehat{S}} \vec{\hat{\omega}} = \vec{v_{\widehat{S}}}$ where $M^{nk}_{\widehat{S}}$ denotes the matrix consisting of all rows corresponding to coalitions in $\widehat{S}$. This is feasible since there exist $\vec{\omega}$ s.t $M^{nk} \vec{\omega} = \vec{v}$. Note that this step assuming feasibility relies on the CGA model being of order $k$ or higher; if not, $M^{nk}_{\widehat{S}} \vec{\hat{\omega}} = \vec{v_{\widehat{S}}}$ may not be feasible. 

In both parts of the proposition below, we will appeal to uniform convergence results to show that this construction yields a $\vec{\hat{\omega}}$ and the corresponding $\vec{\hat{v}}$ such that it approximates $\vec{v}$ with high probability. In all the sample complexity results that follow, let $c$ denote a generic constant.

\begin{prop}
\label{prop: pac}
Suppose $\vec{v}$ is a $k$th order CGA model with parameter vector $\omega$ of bounded $\ell_1$ norm. Then, with a set $\widehat{S}$ of $(C, v(C))$ data points of size $m \geq c \left ( \frac{d_k + \log(1 / \Delta)}{\delta^2} \right )$ uniformly sampled from $2^A$, we may compute $\vec{\hat{\omega}}$ and its corresponding $\vec{\hat{v}}$ as above such that, with probability at least $1 - \Delta$ over the samples $\widehat{S}$:

$$\Pr_{C \sim 2^A} [ \hat{v}(C) = v(C)] \geq 1- \delta$$

\end{prop}

\begin{proof}

The proof is motivated by the observation that $\vec{\omega}$ may be viewed as a linear classifier with dimension $d_k$. Indeed, if $\vec{\omega}$ is the true weight, then $M^{nk} \vec{\omega} = \vec{v}$, which is equivalent to:

$$ [-M^{nk}_C, v(C)]^T [\vec{\omega}, 1] \geq 0 \text{ and } [M^{nk}_C, -v(C)]^T [\vec{\omega}, 1] \geq 0 \text{ for all C}$$

where $M^{nk}_C$ denotes the row of $M_{nk}$ corresponding to coalition $C$ and $[\vec{a},\vec{b}], \vec{a} \in \mathbb{R}^n, \vec{b} \in \mathbb{R}^m$ denotes the $n + m$-dimensional vector obtained by concatenation of $\vec{a}$ and $\vec{b}$. Thus, if we define a classification task with $2 \cdot 2^N$ data points that have features $[-M^{nk}_C, v(C)]$, $[M^{nk}_C, -v(C)]$ and labels $1$ for all the points, we know there exists a  classifier $f(\vec{x}) = \text{sign}([\vec{\omega}, 1]^T \vec{x})$ which achieves zero loss; here we take the sign of $0$ to be $1$.

Define data distribution $\mathcal{D}$ to be the uniform distribution over these $2 \cdot 2^N$ data points. A draw of size $m$ from $\mathcal{D}$ may be simulated by sampling coalitions from the uniform distribution over $2^A$ and then for each chosen coalition $C$, randomly choosing between $[-M^{nk}_C, v(C)]$ and $[M^{nk}_C, -v(C)]$ with equal probability.

Now, we are ready to prove that the $\hat{v}$ satisfies the statement in Proposition \ref{prop: pac}. To do this, we use the uniform convergence result below (Theorem 6.8 from ~\cite{shalev2014understanding}): 

\begin{lemma}
\label{lem: uniform_convergence}
Let $\mathcal{H}$ be a hypothesis class for the classifier, and let $f$ be the true underlying classifier. If $\mathcal{H}$ has VC-dimension $d$, then with 
$$
m \geq c \left(\frac{d + \log\left(\frac 1 \Delta\right)}{\delta^2}\right)
$$ 
i.i.d data points $\vec{x}_1,...,\vec{x}_m \sim \Dcal$, 
$$\delta \geq \left|\Pr_{\vec{x} \sim D}[h(\vec{x}) \neq f(\vec{x})] - \frac{1}{m} \sum_{i=1}^m \mathds{1}_{h(\vec{x}^i) \neq f(\vec{x}^i)}\right|$$
for all $h \in \mathcal{H}$ and with probability $1 - \Delta$ over the sampled data points.
\end{lemma}

By construction, the classifier defined by $\vec{\hat{\omega}}$, $h(x) = \text{sign}([\vec{\hat{\omega}}, 1]^Tx))$, achieves zero empirical risk on $\hat{S}$ since $h(\vec{x}^i) = 1 = f(\vec{x}^i)$. So, we apply the uniform convergence result Lemma \ref{lem: uniform_convergence} with $\delta / 2$ to get that with probability $1 - \Delta$ over the sampled data points $\vec{x}$ from $\mathcal{D}$:

\begin{align*}
    \frac{\delta}{2} & \geq \Pr_{\vec{x} \sim D}[h(\vec{x}) \neq f(\vec{x})] \\
    & = \Pr_{\vec{x} \sim D} [[\vec{\hat{\omega}}, 1]^T \vec{x} < 0)] \\
    & = \frac{1}{2} \Pr_{C \sim 2^A} [M^{nk}_C \vec{\hat{\omega}} > v_C] + \frac{1}{2} \Pr_{C \sim 2^A} [M^{nk}_C \vec{\hat{\omega}} < v_C]\\
    & = \frac{1}{2} \left( 1 - \Pr_{C \sim 2^A} [M^{nk}_C \vec{\hat{\omega}} = v_C] \right)
\end{align*}

Therefore, the guarantee for $\vec{\hat{\omega}}$ over distribution $\mathcal{D}$ translates to the guarantee over the uniform distribution $2^A$ that $\hat{v} = M^{nk} \vec{\hat{\omega}}$ can overpredict or underpredict for at most $\delta$ percent of all coalitions.

To finish, we note that $[\vec{\omega}, 1]$ belongs to the hypothesis class of linear classifiers of dimension $d_k + 1$, which is known to have VC Dimension $d_k + 1$. So $\mathcal{H} = \{[\vec{\omega}, 1] \mid \vec{\omega} \in \mathbb{R}^{d_k} \}$ has VC dimension $d \leq d_k + 1$. And so, our sample complexity needed for $\vec{\hat{\omega}}$ to attain small generalization risk using Lemma \ref{lem: uniform_convergence} is $O(\frac{d_k + \log(1 / \Delta)}{\delta^2})$.

\end{proof}

We remark that the sample complexity needed is on the same order as that shown by Theorem 1 in section \ref{app: id_theorems} 

Next, we provide a PMAC-like guarantee \cite{balcan2011learning} with much smaller sample complexity.

\begin{prop}

With samples of size $m \geq c \left( \frac{\log(d_k) + \log(1 / \Delta)}{\epsilon^2 \delta^2} \right)$ uniformly sampled from $2^A$, we may compute $\vec{\hat{\omega}}$ and its corresponding $\vec{\hat{v}}$ as above such that, with probability at least $1 - \Delta$ over the samples:

$$\Pr_{C \sim 2^A}\left[   (1 - \epsilon) \hat{v}(C) \leq v(C) \leq (1 + \epsilon) \hat{v}(C)\right] \geq 1- \delta$$
\end{prop}

\begin{proof}

The proof follows from combining two known theorems adapted to our setting.

The left hand side of the probabilistic guarantee follows from a straightforward adaptation of the proof of Theorem 5 in \cite{balkanski2017statistical}. In particular, the only tweak to the proof is that the features of the data points are to be instantiated as $M^{nk}_C / v(C)$ instead of $\mathds{1}_C / v(C)$. Since $\vec{\omega}$ is bounded by our assumption, we do not need to bound it in terms of values of $v(C)$'s as is done in the proof of \cite{balkanski2017statistical}. We note the loss function would then be defined as $\ell(\vec{\omega}, (M^{nk}_C / v(C), y)) = [\frac{M^{nk}_C \vec{\omega}}{v(C)} - 1]_{+}$ and $\vec{\hat{\omega}}$ achieves zero empirical loss because $M^{nk}_{\widehat{S}} \vec{\hat{\omega}} = \vec{v_{\widehat{S}}} \Rightarrow M^{nk}_C \vec{\hat{\omega}} = v_C \Rightarrow \frac{M^{nk}_C \vec{\hat{\omega}}}{v(C)} - 1 = 0$ for all $C \in S$. Altogether, we may arrive at the statement below:

\textit{With a set of $m \geq c \left(\frac{\log(d_k) + \log(1 / \Delta)}{\epsilon^2 \delta^2} \right)$ coalitions uniformly sampled from $2^A$, $\vec{\hat{\omega}}$ constructed as above is such that:}

$$\Pr_{C \sim 2^A}\left[  (1 - \epsilon) M^{nk}_C \vec{\hat{\omega}} \leq v(C)  \right] \geq 1- \delta$$

\textit{with probability at least $1 - \Delta$ over the samples.}

The right hand side follows from a related theorem, Theorem 2 in \cite{yan2020if} with the same change in the data features. Again, we can verify that $\vec{\hat{\omega}}$ achieves zero empirical loss:

\textit{With a set of $m \geq c \left( \frac{\log(d_k) + \log(1 / \Delta)}{\epsilon^2 \delta^2} \right)$ coalitions uniformly sampled from $2^A$, $\vec{\hat{\omega}}$ constructed as above is such that:}

$$\Pr_{C \sim 2^A}[ v(C) \leq (1 + \epsilon) M^{nk}_C \vec{\hat{w}} ] \geq 1- \delta$$

\textit{with probability at least $1 - \Delta$ over the samples.}

With this, we can initialize both theorems with $\Delta / 2$ and $\delta / 2$. We first union bound over the random draw of $m-$ size samples to conclude that with probability $\geq 1 - \Delta$, both inequalities hold for $\vec{\hat{\omega}}$, meaning that by union bound again for the random draw of $C$ over $2^A$:

$$\Pr_{C \sim 2^A} [ (1 - \epsilon) M^{nk}_C \vec{\hat{w}} \leq v(C) \leq (1 + \epsilon) M^{nk}_C \vec{\hat{w}} ] \geq 1- \delta$$

\end{proof}

In summary, this means that under an even looser definition of approximability of the CGA model, the sample complexity needed is much smaller: only $O(\log (d_k))$ number of points are needed. Since $d_k \leq 2^n$, this means at most $O(n)$ samples are needed to estimate \emph{most} of the coalition values \emph{approximately} with high probability.

\textbf{Remark:} more generally, we may obtain the above two guarantees under the same sample complexity for any setting where we are looking to estimate solutions $\vec{x}$ to large scale linear programs $A\vec{x} = \vec{b}$, knowing apriori that $\| x \|_1$ is bounded. In such cases, we may obtain a PAC and PMAC-like result by computing $\vec{\hat{x}}$ from randomly sampled constraints $\vec{a_i}^T \vec{x} = b_i$. Notice here that $A \in \mathbb{R}^{2^n \times d_k}$ and the PMAC notion avoids needing the exponential sample complexity that is required to construct $\vec{b}$ to compute an exact solution. 

This result may be of independent interest.

\subsection{Shapley Noise Bound Theorem Proofs}
\label{app: noise_bound_theorems}

\begin{theorem}[Shapley noise L2 bound]

The L2 norm of the estimation error of the Shapley values is bounded by:
 
 \begin{equation}\label{eq:4a}
  \sum_{i=1}^n \left ( \varphi_{i}(v) -  \varphi_i(\hat{v}) \right )^2 \leq \frac{2}{n}\sum_{C \in 2^A} \left (v(C) - \hat{v} (C) \right)^2
 \end{equation}

\end{theorem}

\begin{proof}

First we observe that the Shapley value is a linear map $\mathbb{R}^{2^n} \rightarrow \mathbb{R}^n$ taking $v$ to $\varphi(v)$. We may describe this map with matrix $S_n \in \mathbb{R}^{n \times 2^n}$ where $n$ is the number of players in the cooperative game. Our work extends a line of work that studies properties of $S_n$, including~\cite{beal2019games} that studies its nullspace.

We have that:

$$|| \varphi(v) - \varphi(\hat{v})||_2 = ||S_n v - S_n \hat{v}||_2 \leq ||S_n||_{op} ||v - \hat{v}||_2$$

It suffices then to obtain the operator norm of $S_n$. We know that $||S_n||_{op} = \sqrt{\sigma_{\max}(S_n^TS_n)}$. $S_n^TS_n$ is complicated to analyze, so we opt to analyze $\sigma_{\max}(S_n S_n^T)$ since we know that the nonzero eigenvalues of $S_n^TS_n$ are the same as those of $S_nS_n^T$. $S_nS_n^T$ has nice structure in that all its off-diagonal entries are the same and all its diagonal entries are the same.

Take the $i$th row of $S_n$, $(S_n)_i$, we know that the entry in this row corresponding to subset $S$ is:

\begin{enumerate}
	\item $\frac{1}{n} \binom{n - 1}{|S| - 1}^{-1}$ if $i \in S$
	\item  $-\frac{1}{n} \binom{n - 1}{|S|}^{-1}$ if $i \not \in S$
\end{enumerate}

Therefore, let $d_1$ denote its diagonal entries, then:

\begin{align*}
	d_1 &= (S_n)_i^T  (S_n)_i \\
		& = \sum_{S \in 2^{[n]}, i \in S}  ( \frac{1}{n} \binom{n - 1}{|S| - 1}^{-1})^2 + \sum_{S \in 2^{[n]}, i \not \in S} (- \frac{1}{n} \binom{n - 1}{|S|}^{-1})^2 \\
		& = \frac{1}{n^2} \sum_{k=1}^n \binom{n-1}{k-1} \binom{n-1}{k-1}^{-2} +  \frac{1}{n^2} \sum_{k=0}^{n-1} \binom{n-1}{k} \binom{n-1}{k}^{-2}
\end{align*}

Let $d_2$ denote its off-diagonal entries. Consider the $i,j$th entry of $S_nS_n^T$, we can characterize the weights in the dot product as follows:

\begin{enumerate}
	\item $(\frac{1}{n} \binom{n - 1}{|S| - 1}^{-1})^2$ if $i, j \in S$
	\item $(\frac{1}{n} \binom{n - 1}{|S| - 1}^{-1})(-\frac{1}{n} \binom{n - 1}{|S|}^{-1})$ if $i \in S$, $j \not \in S$
	\item $(-\frac{1}{n} \binom{n - 1}{|S|}^{-1})(\frac{1}{n} \binom{n - 1}{|S| - 1}^{-1})$ if $i \not \in S$, $j \in S$
	\item  $(-\frac{1}{n} \binom{n - 1}{|S|}^{-1})^2$ if $i, j \not \in S$
\end{enumerate}

Therefore, when we sum these together:

\begin{align*}
	d_2 &= (S_n)_i^T  (S_n)_j \\
		& = \sum_{S \in 2^{[n]}, i, j  \in S}  ( \frac{1}{n} \binom{n - 1}{|S| - 1}^{-1})^2 - \sum_{S \in 2^{[n]}, i \in S, j \not \in S} (\frac{1}{n} \binom{n - 1}{|S| - 1}^{-1})(-\frac{1}{n} \binom{n - 1}{|S|}^{-1}) \\
		& \qquad \qquad - \sum_{S \in 2^{[n]}, i \not \in S, j \in S} (-\frac{1}{n} \binom{n - 1}{|S|}^{-1})(\frac{1}{n} \binom{n - 1}{|S| - 1}^{-1}) +  \sum_{S \in 2^{[n]}, i, j  \not \in S} (- \frac{1}{n} \binom{n - 1}{|S|}^{-1})^2 \\
		& = \sum_{k=2}^n \binom{n-2}{k-2} \binom{n-1}{k-1}^{-2} - 2\sum_{k=1}^{n-1} \binom{n-2}{k-1} \binom{n-1}{k}^{-1} \binom{n-1}{k-1}^{-1} + \sum_{k=0}^{n-2} \binom{n-2}{k} \binom{n-1}{k}^{-2}
\end{align*}

It's easy to check that $d_1 > d_2$ since $d_2 =  (S_n)_i^T  (S_n)_j \leq ||(S_n)_i||_2 ||(S_n)_j||_2 = (S_n)_i^T  (S_n)_i = d_1$. 

And so, we may write:

$$S_nS_n^T = (d_1 - d_2) I_{n} + d_2 1_n$$

where $1_n$ is the all ones matrix.

This allows us to characterize all the eigenvalues of $S_nS_n^T$ and in particular the biggest one. 

If the SVD of $1_n = U  D U^T$, then we know that $D$ is a diagonal matrix with one entry being $n$ as this is an eigenvalue of $1_n$ and the rest being $0$ since $1_n$ is only rank $1$. And so, $$S_nS_n^T = U[(d_1 - d_2) I_{n} + d_2 D] U^T$$

This means that the top eigenvalue is $d_1 - d_2 + n \cdot d_2$ and the rest are all $d_1 - d_2$.

Evaluating $d_1 - d_2 + n \cdot d_2 = d_1 + (n-1) d_2$:

\begin{align*}
	d_1 + (n-1) d_2 & = \frac{1}{n^2} (\sum_{k=2}^{n-2} \binom{n-1}{k-1}^{-1} + \binom{n-1}{k}^{-1} \\
	& \qquad + (n-1) [\frac{k-1}{n-1} \binom{n-1}{k-1}^{-1} - 2 \frac{k}{n-1} \binom{n-1}{k-1}^{-1} + \binom{n-2}{k} \binom{n-1}{k}^{-2}]) + \frac{1}{n^2} r \\
	& =  \frac{1}{n^2} ((\sum_{k=2}^{n-2} k  \binom{n-1}{k-1}^{-1} +  \binom{n-1}{k}^{-1} - 2k  \binom{n-1}{k-1}^{-1} + (n-1)  \binom{n-2}{k} \binom{n-1}{k}^{-2}) + \frac{1}{n^2} r \\
	& =  \frac{1}{n^2} ((\sum_{k=2}^{n-2} -k \binom{n-1}{k-1}^{-1} + \binom{n-1}{k}^{-1} + (n - 1 -k )\binom{n-1}{k}^{-1}) + \frac{1}{n^2} r \\
	& =  \frac{1}{n^2} ((\sum_{k=2}^{n-2} - \frac{k! (n-k)!}{(n-1)!} + (n-k) \frac{k! (n-1-k)!}{(n-1)!}) + \frac{1}{n^2} r \\
	&= \frac{1}{n^2} r
\end{align*}

It just remains to evaluate $r$ which are the residual terms from the sums, they are:

\begin{align*}
	r & =[1 + 1 + \frac{1}{n-1}] + [1 + \frac{1}{n-1} + 1] \text{ (from the two sums in } d_1) \\
	& \qquad + (n - 1) ( ( [1 + \frac{n-2}{(n-1)^2}] - 2[\frac{1}{n-1} + \frac{1}{n-1}] + [1 + \frac{(n-2)}{(n-1)^2}]) \text{ (from the three sums in } d_2) \\
	& = 4 + \frac{2}{n-1} + 2n - 2 -4 + \frac{2(n-2)}{n-1} \\
	& = 2n
\end{align*}

To summarize, we get that:

$$\sigma_{\max}(S_n S_n^T) = d_1 + (n-1) d_2 = \frac{1}{n^2} 2n = \frac{2}{n}$$

$$\Rightarrow ||S_n||_{op} = \sqrt{\sigma_{\max}(S_n^T S_n)} =  \sqrt{\sigma_{\max}(S_n S_n^T)} = \sqrt{ \frac{2}{n}}$$

which proves that $|| \varphi(v) - \varphi(\hat{v})||_2 \leq \sqrt{\frac{2}{n}} ||v - \hat{v}||_2$ \eqref{eq:4a}, as desired.

\end{proof}

The above is a worst case analysis by computing the largest singular value of the Shapley matrix. It turns out, most singular values of the Shapley matrix are very small and won't lead to a large amplification of the noise in the characteristic function.

We perform average case analysis by assuming that the error in the characteristic function is drawn uniformly from a smooth distribution, which is not very "peaky" anywhere, over all noise $v - \hat{v}$ with the same L2 norm.

\begin{theorem}
Assuming that $v - \hat{v}$ is drawn from distribution $\mathcal{D}_{B_r}$ with support equal to a sphere and smooth in that $\kappa_0 \leq \Pr_{\mathcal{D}_{B_r}}(x) \leq \kappa_1$ for any point $x$ in its support, then:

 \begin{equation}
  \mathbb{E}_{v - \hat{v} \sim \mathcal{D}_{B_r}}[\| \varphi(v) -  \varphi(\hat{v}) \|_2^2] \leq \frac{6}{n} \frac{\kappa_1}{\kappa_0} \frac{\| v - \hat{v} \|_2^2}{2^n}
 \end{equation}
\end{theorem}

\begin{proof}

To obtain the bound in the theorem, we first prove the lemma below:

\begin{lemma}
\label{lem: average_singular_value}

Let $\mathcal{D}_{B_r}$ be a distribution with support equal to a sphere with radius $r$
and smooth in that $\kappa_0 \leq \Pr_{\mathcal{D}_{B_r}}(\vec{x}) \leq \kappa_1$ for any $\vec{x}$ in its support.
Consider any 
matrix $A \in \mathbb{R}^{m_1 \times m_2}$:

$$\mathbb{E}_{\vec{x} \sim \mathcal{D}_{B_r}}[\| A \vec{x} \|_2^2] \leq \frac{\kappa_1}{\kappa_0} \frac{\Tr(A^TA)}{m_2} \cdot r^2$$
\end{lemma}

\begin{proof}
Since $A^TA$ is symmetric and thus diagonalizable, consider its $m_2$ orthonormal eigenvectors $\vec{u_1}, .., \vec{u_{m_2}}$. We know that $\vec{u_1}, .., \vec{u_{m_2}}$ forms a basis of $\mathbb{R}^{m_2}$ and we can then write any $\vec{x}$ in the support of $\mathcal{D}_{B_r}$ as $\sum_{j=1}^{m_2} \alpha_{j} \vec{u_j}$. Moreover, 

$$r^2 = \| \vec{x} \|^2 = (\sum_{j=1}^{m_2} \alpha_{j} \vec{u_j})^T(\sum_{j=1}^{m_2} \alpha_{j} \vec{u_j}) = \sum_{j=1}^{m_2} \alpha_{j}^2$$

since $\vec{u_j}^T \vec{u_i} = 0$ for $i \neq j$ and $ \| \vec{u_j} \|_2^2 = 1$. 

Define $\mathcal{D}'_{B_r}$ to be the distribution over $\vec{\alpha}$ that corresponds to each $\vec{x}$ drawn from $\mathcal{D}_{B_r}$ and set $S_{D'}$ be its support (which may be characterized as a $m_2$ dimensional standard simplex as defined by $(\alpha_1^2 / r^2, ..., \alpha_{m_2}^2 / r^2)$. We abuse notation in letting $x(\vec{\alpha})$ be the corresponding $x$ to coefficients vector $\vec{\alpha}$. It's a 1-1 correspondence, and so from the smoothness assumption on $\mathcal{D}_{B_r}$,  $\Pr_{\mathcal{D}'_{B_r}}(\vec{\alpha}) = \Pr_{\mathcal{D}_{B_r}}(x(\vec{\alpha})) \in [\kappa_0, \kappa_1]$.

Define $k^* = \argmax_{k \in [m_2]} \mathbb{E}_{\vec{\alpha} \sim \mathcal{D}'_{B_r}}[\alpha_{k}^2]$, then for any $i \neq k^*$:

\begin{align*}
    \mathbb{E}_{\vec{\alpha} \sim \mathcal{D}'_{B_r}}[\alpha_{k^*}^2] & = \int_{S_{D'}} \alpha_{k^*}^2 \Pr_{\mathcal{D'}_{B_r}}(\vec{\alpha}) d \vec{\alpha} \\
    & \leq \int_{S_{D'}} \alpha_{k^*}^2 \kappa_1 d \vec{\alpha}\\
    & = \int_{S_{D'}} \alpha_i^2 \kappa_1 d \vec{\alpha}\\
    & \leq \int \alpha_i^2  \frac{\kappa_1}{\kappa_0} \Pr_{\mathcal{D'}_{B_r}}(\vec{\alpha}) d \vec{\alpha} \\
    & =  \frac{\kappa_1}{\kappa_0} \mathbb{E}_{\vec{\alpha} \sim \mathcal{D}'_{B_r}}[\alpha_{i}^2]
\end{align*}
where the second equality follows from the symmetry of the support of $\mathcal{D}_{B_r}$, which is a sphere.

This implies that:

$$\mathbb{E}_{\vec{\alpha} \sim \mathcal{D}'_{B_r}}[\alpha_{k^*}^2]  \leq \frac{\sum_{j=1}^{m_2} \frac{\kappa_1}{\kappa_0} \mathbb{E}_{\vec{\alpha} \sim \mathcal{D}'_{B_r}}[\alpha_{j}^2 ]}{m_2} = \frac{\kappa_1}{\kappa_0} \frac{\mathbb{E}_{\vec{\alpha} \sim \mathcal{D}'_{B_r}}[\sum_{j=1}^{m_2} \alpha_{j}^2 ]}{m_2} = \frac{\kappa_1}{\kappa_0} \frac{r^2}{m_2}$$

Therefore:
\begin{align*}
    \mathbb{E}_{\vec{x} \sim \mathcal{D}_{B_r}}[\| A \vec{x} \|_2^2] & = \mathbb{E}_{\vec{x} \sim \mathcal{D}_{B_r}}[\vec{x}^T A^TA \vec{x}] \\
    & = \mathbb{E}_{\vec{x} \sim \mathcal{D}_{B_r}}[\vec{x}^T (\sum_{j=1}^{m_2} \alpha_{j}  \lambda_j \vec{u_j})] \\
    & = \mathbb{E}_{\vec{\alpha} \sim \mathcal{D}'_{B_r}}[\sum_{j=1}^{m_2} \lambda_j \alpha_{j}^2 ] \\
    & = \sum_{j=1}^{m_2} \lambda_j \mathbb{E}_{\vec{\alpha} \sim \mathcal{D}'_{B_r}}[\alpha_{j}^2 ] \\
    & \leq \mathbb{E}_{\vec{\alpha} \sim \mathcal{D}'_{B_r}}[\alpha_{k^*}^2 ] \sum_{j=1}^{m_2} \lambda_j \\
    & \leq  \frac{\kappa_1}{\kappa_0} \frac{r^2}{m_2} \sum_{j=1}^{m_2} \lambda_j
\end{align*}

\end{proof}

Using this Lemma \ref{lem: average_singular_value}, we can then perform an average case analysis:

$$\mathbb{E}[ \| S_n \vec{x} \|_2^2] \leq \frac{\kappa_1}{\kappa_0} \frac{\Tr(S_n^TS_n)}{2^n} \| \vec{x} \|_2^2 = \frac{\kappa_1}{\kappa_0} \frac{\Tr(S_nS_n^T)}{2^n} \| \vec{x} \|_2^2$$

We know that $\Tr(S_nS_n^T) = nd_1$ so the average case multiplier of the noise is:

\begin{align*}
    d_1 & = \frac{1}{n^2} \sum_{k=1}^n \binom{n-1}{k-1}^{-1} +  \frac{1}{n^2} \sum_{k=0}^{n-1} \binom{n-1}{k}^{-1} \\
    & = \frac{1}{n^2} (2 + \sum_{k=1}^{n-1} \binom{n-1}{k-1}^{-1} + \binom{n-1}{k}^{-1}) \\
    & = \frac{2}{n^2} + \frac{1}{n^2}(\sum_{k=1}^{n-1} \frac{(k-1)!(n-k-1)!(k + n-k)}{(n-1)!})\\
    & = \frac{2}{n^2} + \frac{1}{n(n-1)}(\sum_{k=1}^{n-1} \binom{n-2}{k-1}^{-1}) \\
    & = \frac{2}{n^2} + \frac{2}{n(n-1)} + \frac{1}{n(n-1)}(\sum_{k=2}^{n-2} \binom{n-2}{k-1}^{-1}) \\
    & \leq \frac{2}{n^2} + \frac{2}{n(n-1)} + \frac{1}{n(n-1)}((n - 3) \binom{n-2}{1}^{-1}) \\
    & = \frac{2}{n^2} + \frac{3n-7}{n(n-1)(n-2)} \\
    & \leq \frac{6}{n^2}
\end{align*}

So, the multiplier is $\frac{\kappa_1}{\kappa_0} \frac{6}{n2^n}$ over the distribution $\mathcal{D}_{B_r}$.
\end{proof}

Next, we can obtain a more general result by integrating across all L2 norms $r$ that $v - \hat{v}$ can take.

\begin{corr}
Suppose noise $v - \hat{v} \sim \mathcal{D}_n$ is such that its conditional distribution satisfies $\kappa_0(r) \leq \Pr_{\mathcal{D}_n}(x | \| x \|_2^2 = r^2) \leq \kappa_1(r)$ for all $r$ and $x$ in $\mathcal{D}_n$'s support, then:

$$\mathbb{E}_{v - \hat{v} \sim \mathcal{D}_n}[\| \varphi(v) -  \varphi(\hat{v}) \|_2^2] \leq \frac{6}{n} \mathbb{E}_r \left[ \frac{\kappa_1(r)}{\kappa_0(r)} \left ( \frac{r^2}{2^n} \right) \right]$$
\end{corr}

\begin{proof}

This follows from iterated expectation:

\begin{align*}
    \mathbb{E}_{v - \hat{v} \sim \mathcal{D}}[\| \varphi(v) -  \varphi(\hat{v}) \|_2^2] & = \mathbb{E}_r [\mathbb{E}_{v - \hat{v} \sim \mathcal{D}_{B_r}} [ \| \varphi(v) -  \varphi(\hat{v}) \|_2^2 \mid \| v - \hat{v} \|_2^2 = r^2] ] \\
    & \leq \mathbb{E}_r \left [ \left ( \frac{\kappa_1(r)}{\kappa_0(r)} \frac{6}{n2^n} \right ) r^2 \right ] \\
    & = \frac{6}{n 2^n} \mathbb{E}_r \left [ \frac{\kappa_1(r)}{\kappa_0(r)} r^2 \right]
\end{align*}

where the inequality holds by Theorem \ref{thm: l2_average_case}.

\end{proof}

\textbf{Remark:} Therefore, if $\mathbb{E}_r [ \frac{\kappa_1(r)}{\kappa_0(r)} r^2] = c \mathbb{E}_r [r^2]$ for some constant $c = O(1)$, then the error in the Shapley value is fairly small and proportional to $O(1 / n)$ of the average L2 error of $v - \hat{v}$.

\begin{theorem}[Shapley noise L1 bound]

The sum of absolute errors in Shapley values is bounded by:

\begin{equation} \label{eq:3a}
  \sum_{i=1}^n \left| \varphi_{i}(v) -  \varphi_i(\hat{v}) \right| \leq \sum_{C \in 2^A} \left| v(C) - \hat{v} (C) \right|
 \end{equation}

Assuming there is no error in estimating the grand coalition nor the empty set and $n \geq 3$, then we can give a stronger bound on the sum of absolute errors:

\begin{equation} \label{eq:3b}
\sum_{i=1}^n \left| \varphi_{i}(v) -  \varphi_i(\hat{v}) \right|  \leq \frac{2}{n} \sum_{C \in 2^A} \left| v(C) - \hat{v} (C) \right|
 \end{equation}

Furthermore, assume players are divided into m equal sized teams, $G_1, ..., G_m$, where $|G_i| = N / m$. 
Then if we compute their Shapley values just with respect to their own teams we get:

\begin{equation} \label{eq:3c}
\sum_{i=1}^n \left| \varphi_{i}(v) -  \varphi_i(\hat{v}) \right|  \leq \frac{2m}{n} \sum_{C \in 2^A} \left| v(C) - \hat{v} (C) \right|
 \end{equation}

\end{theorem}

\begin{proof}
We can express the difference in Shapley value for $i$ as:

\begin{align*}
    |\varphi_i(v) - \varphi_i(\hat{v}) |
    & = | \frac{1}{n} \sum_{S \subseteq [n] \char`\\ \{i\}} \binom{n-1}{|S|}^{-1} ([v(S \cup \{i\}) - \hat{v}(S \cup \{i\})] - [v(S) - \hat{v}(S)]) | \\ 
    & \leq \frac{1}{n} \sum_{S \subseteq [n] \char`\\ \{i\}} \binom{n-1}{|S|}^{-1} (|v(S \cup \{i\}) - \hat{v}(S \cup \{i\})| + |v(S) - \hat{v}(S)|) \\
    & = \frac{1}{n} \sum_{s=0}^{n-1} \binom{n-1}{s}^{-1} \sum_{S \subseteq [n] \char`\\ \{i\}, |S| = s} (|v(S \cup \{i\}) - \hat{v}(S \cup \{i\})| + |v(S) - \hat{v}(S)|)
\end{align*}

Thus, for any $S$ of size $s$:

\begin{itemize}
    \item If it contains element $i$, its L1 $v$ error is weighted by $\binom{n-1}{s - 1}^{-1}$. 
    \item If it doesn't, it is weighted by $\binom{n-1}{s}^{-1}$.
\end{itemize}

Observe that the unweighted RHS is equal to:

\begin{align*}
& = \frac{1}{n} \sum_{s=0}^{n-1} \sum_{S \subseteq [n] \char`\\ \{i\}, |S| = s} (|v(S \cup \{i\}) - \hat{v}(S \cup \{i\})| + |v(S) - \hat{v}(S)|) \\
& =  \frac{1}{n} \sum_{S \subseteq [n] \char`\\ \{i\}} |v(S \cup \{i\}) - \hat{v}(S \cup \{i\})| + \sum_{S \subseteq [n] \char`\\ \{i\}} |v(S) - \hat{v}(S)| \\
& =  \frac{1}{n} ||v - \hat{v}||_1
\end{align*}

Therefore, since $\binom{n-1}{s}^{-1} \leq 1$ for $s \in [0, n- 1]$:

\begin{align*}
|\varphi_i(v) - \varphi_i(\hat{v}) | & \leq \frac{1}{n} \sum_{s=0}^{n-1} \binom{n-1}{s}^{-1} \sum_{S \subseteq [n] \char`\\ \{i\}, |S| = s} (|v(S \cup \{i\}) - \hat{v}(S \cup \{i\})| + |v(S) - \hat{v}(S)|) \\ 
& \leq \frac{1}{n} ||v - \hat{v}||_1
\end{align*}

Summing across all $i$'s, this proves inequality \eqref{eq:3a}.

Note that $\binom{n-1}{s}^{-1}=1$ holds only for (i) the full set $[n]$ ($s = n - 1$) (ii) the set $\{i\}$ ($s = 0$) (iii) empty set $\O$ ($s = 0$) (iv) set $[n] \char`\\ \{i\} = [-i] $ ($s = n-1$). Thus we can obtain equality if all of the errors in $v$ lie in estimating the full set or the empty set. This makes the bound tight.

We obtain a stronger inequality \eqref{eq:3b} if we assume that there is no error in estimating the empty nor the grand coalition value:

Let $e = ||v - \hat{v}||_1$ and $e_i = | v(\{i\}) - \hat{v}(\{i\}) | + | v([-i]) - \hat{v}([-i]) |$. Then:

\begin{align*}
     | \varphi_i(v) - \varphi_i(\hat{v}) |  & \leq 
     \frac{1}{n} \sum_{s=0}^{n-1} \binom{n-1}{s}^{-1} \sum_{S \subseteq [n] \char`\\ \{i\}, |S| = s} (|v(S \cup \{i\}) - \hat{v}(S \cup \{i\})| + |v(S) - \hat{v}(S)|) \\
     & \leq \frac{1}{n}  e_i  + \frac{1}{n} \sum_{s=1}^{n-2} \binom{n-1}{s}^{-1} \sum_{S \subseteq [n] \char`\\ \{i\}, |S| = s} (|v(S \cup \{i\}) - \hat{v}(S \cup \{i\})| + |v(S) - \hat{v}(S)|) \\
     & \leq \frac{1}{n} e_i + \frac{1}{n(n-1)} (e - e_i)
\end{align*}

since $\binom{n-1}{s}^{-1} \leq \frac{1}{n - 1}$ for $s \in [1, n - 2]$.

Summing this across i gives:

\begin{align*}
    | \varphi(v) - \varphi(\hat{v}) | & \leq \frac{1}{n} \sum_{i=1}^n e_i +  \frac{1}{n(n-1)} (ne - \sum_{i=1}^n e_i) \\
    & = \frac{e}{n - 1}  +  \frac{n - 2}{n(n-1)} (\sum_{i=1}^n e_i) \\
    & \leq \frac{e}{n - 1} + \frac{n - 2}{n(n-1)} e \\
    & = \frac{2e}{n}
\end{align*}

since $\sum_{i=1}^n e_i \leq e$.

This proves inequality \eqref{eq:3b}.

In some case, as is the case with our NBA experimental setup, players are divided into $m$ groups, $G_1, ..., G_m$, and we wish to compute their Shapley values only with respect to their own groups. We can follow a similar analysis as above to derive a bound on the Shapley values. For player $i \in G_j$:

\begin{align*}
    |\varphi_{i}(v) - \varphi_i(\hat{v}) |
    & = | \frac{1}{|G_j|} \sum_{S \subseteq G_j \char`\\ \{i\}} \binom{|G_j| -1}{|S|}^{-1} ([v(S \cup \{i\}) - \hat{v}(S \cup \{i\})] - [v(S) - \hat{v}(S)]) | \\
    & \leq \frac{1}{|G_j|} \sum_{S \subseteq G_j \char`\\ \{i\}} \binom{|G_j|-1}{|S|}^{-1} (|v(S \cup \{i\}) - \hat{v}(S \cup \{i\})| + |v(S) - \hat{v}(S)|) \\
    & = \frac{1}{|G_j|} \sum_{s=0}^{|G_j|-1} \binom{|G_j|-1}{s}^{-1} \sum_{S \subseteq G_j \char`\\ \{i\}, |S| = s} (|v(S \cup \{i\}) - \hat{v}(S \cup \{i\})| + |v(S) - \hat{v}(S)|)
\end{align*}

Let $E_j = \sum_{S \subseteq G_j} |v(S) - \hat{v}(S)|$. It's clear that $\sum_{j=1}^m E_j \leq e$ since any two teams $G_{j_1}, G_{j_2}$ are disjoint for $j_1 \neq j_2$ and thus don't have any subsets in common; note our assumption that the empty set is estimated without any error by $\hat{v}$.

To maximize the cumulative error, all the errors in $e$ should be placed in subsets $S$ with $S \subseteq G_j$ for some $j$. So WLOG we can assume that $\sum_{j=1}^m E_j = e$. Using inequality \eqref{eq:3b}, we get that:

$$\sum_{i \in G_j} | \varphi_i(v) - \varphi_i(\hat{v}) | \leq \frac{2E_j}{|G_j|}$$

So the overall bound is:

$$| \varphi(v) - \varphi(\hat{v}) | \leq \sum_{j=1}^m \frac{2E_j}{|G_j|}$$

Assume $|G_j| = n / m$ for each j, this simplifies to $\frac{2m}{n} e$ and proves inequality \ref{eq:3c}.

\end{proof}

While the above bounds are tight, the analysis is worst case. For instance, for the first bound we provide, equality holds when all the error in the $\|v - \hat{v}\|_1$ vector is in the coalition value of the grand coalition or the empty set. Below, we provide a simple, average case analysis to show that on average, a randomly drawn error vector leads to a small increase in L1 Shapley error in expectation.

\begin{theorem}[Average case Shapley noise L1 bound]
\label{thm: average_l1}

Assuming that the error $\vec{v} - \vec{\hat{v}}$ is such that the vector $| \vec{v} - \vec{\hat{v}}| / r$ (where absolute value is coordinate wise) is drawn from distribution $\mathcal{D}_{S_r}$ with support equal to the surface of a $2^n$-simplex and smooth in that $\kappa_0 \leq \Pr_{\mathcal{D}_{S_r}}(\vec{x}) \leq \kappa_1$ for any point $\vec{x}$ in its support, then:

 \begin{equation}
  \mathbb{E}_{\vec{v} - \vec{\hat{v}} \sim \mathcal{D}_{S_r}}[\| \varphi(\vec{v}) -  \varphi(\vec{\hat{v}}) \|_1] \leq  2 \frac{\kappa_1}{\kappa_0} \frac{\| \vec{v} - \vec{\hat{v}} \|_1}{2^n}
 \end{equation}

\end{theorem}

\begin{proof}

\begin{align*}
    \mathbb{E}_{\vec{v} - \vec{\hat{v}} \sim \mathcal{D}_{S_r}} [ \left|\varphi_i(\vec{v}) -  \varphi_i(\vec{\hat{v}}) \right| ]
    & = \mathbb{E}_{\vec{v} - \vec{\hat{v}} \sim \mathcal{D}_{S_r}} \left  [ \big | \frac{1}{n} \sum_{S \subseteq [n] \char`\\ \{i\}} \binom{n-1}{|S|}^{-1} ([v(S \cup \{i\}) - \hat{v}(S \cup \{i\})] - [v(S) - \hat{v}(S)]) \big | \right ] \\ 
    & \leq \frac{1}{n} \sum_{s=0}^{n-1} \binom{n-1}{s}^{-1} \sum_{S \subseteq [n] \char`\\ \{i\}, |S| = s} \big ( \mathbb{E}_{\vec{v} - \vec{\hat{v}} \sim \mathcal{D}_{S_r}} \big [ [|v(S \cup \{i\}) - \hat{v}(S \cup \{i\})| \big ] + \\
    & \hspace{23em}  \mathbb{E}_{\vec{v} - \vec{\hat{v}} \sim \mathcal{D}_{S_r}} \big [ |v(S) - \hat{v}(S)| \big ] \big ) \\
    & \stackrel{(1)}{\leq} \frac{1}{n} \sum_{s=0}^{n-1} \binom{n-1}{s}^{-1} \sum_{S \subseteq [n] \char`\\ \{i\}, |S| = s} \bigg ( \frac{\kappa_1}{\kappa_0} \frac{r}{2^n} + \frac{\kappa_1}{\kappa_0} \frac{r}{2^n} \bigg ) \\
    & = \frac{2}{n} \frac{\kappa_1}{\kappa_0} \frac{r}{2^n}
\end{align*}

where $(1)$ is due to the following: 

Let subset $C^* = \argmax_{C} \mathbb{E}_{\vec{v} - \vec{\hat{v}} \sim \mathcal{D}_{S_r}} [ [|v(C) - \hat{v}(C)|] $ and subset $C' = \argmin_{C} \mathbb{E}_{\vec{v} - \vec{\hat{v}} \sim \mathcal{D}_{S_r}} [ [|v(C) - \hat{v}(C)|] $:

\begin{align*}
   \mathbb{E}_{\vec{v} - \vec{\hat{v}} \sim \mathcal{D}_{S_r}} [ [|v(C^*) - \hat{v}(C^*)|] & = \int |v(C^*) - \hat{v}(C^*)| \Pr_{\mathcal{D}_{S_r}}(\vec{v} - \vec{\hat{v}}) d (\vec{v} - \vec{\hat{v}}) \\
    & \leq \int |v(C^*) - \hat{v}(C^*)| \kappa_1 d (\vec{v} - \vec{\hat{v}}) \\
    & \stackrel{(2)}{=} \int |v(C') - \hat{v}(C')| \kappa_1 d (\vec{v} - \vec{\hat{v}}) \\
    & \leq \int |v(C') - \hat{v}(C')| \frac{\kappa_1}{\kappa_0} \Pr_{\mathcal{D}_{S_r}}(\vec{v} - \vec{\hat{v}}) d (\vec{v} - \vec{\hat{v}})  \\
    & =  \frac{\kappa_1}{\kappa_0} \mathbb{E}_{\vec{v} - \vec{\hat{v}} \sim \mathcal{D}_{S_r}} [ [|v(C') - \hat{v}(C')|] \\
    & \stackrel{(3)}{\leq} \frac{\kappa_1}{\kappa_0} \Big( \frac{r}{2^n} \Big).
\end{align*}

Here $(2)$ holds by symmetry as the expectation of any two vector coordinates under a uniform distribution over the simplex of vectors is the same. $(3)$ holds because every vector in the support of $\mathcal{D}_{S_r}$ has L1 norm of $r$, $\sum_{C} \mathbb{E}_{\vec{v} - \vec{\hat{v}} \sim \mathcal{D}_{S_r}} [ |v(C) - \hat{v}(C)| ] = \mathbb{E}_{\vec{v} - \vec{\hat{v}} \sim \mathcal{D}_{S_r}} [ \sum_{C} |v(C) - \hat{v}(C)| ] = r$ and so by our choice of $C'$,  $\mathbb{E}_{\vec{v} - \vec{\hat{v}} \sim \mathcal{D}_{S_r}} [ |v(C') - \hat{v}(C')| ] \leq \frac{r}{2^n}$.

Summing $\mathbb{E}_{\vec{v} - \vec{\hat{v}} \sim \mathcal{D}_{S_r}} [ |\varphi_i(\vec{v}) -  \varphi_i(\vec{\hat{v}}) | ]$ across all $i$ gives the result.

\end{proof}

This means that, on average, for a randomly drawn $\varphi(\vec{v}) -  \varphi(\vec{\hat{v}})$ with a fixed error budget in L1 error, the L1 error in the Shapley is only proportional to the average error in estimating each coalition. 
Next, we can obtain a more general bound by integrating across all L1 norms $r$ that $\varphi(\vec{v}) -  \varphi(\vec{\hat{v}})$ can take.

\begin{corr}
Suppose noise $v - \hat{v} \sim \mathcal{D}_n$ is such that its conditional distribution satisfies $\kappa_0(r) \leq \Pr_{\mathcal{D}_n}(x | \| x \|_1 = r) \leq \kappa_1(r)$ for all $r$ and $x$ in $\mathcal{D}_n$'s support, then:

$$\mathbb{E}_{v - \hat{v} \sim \mathcal{D}_n}[\| \varphi(\vec{v}) -  \varphi(\vec{\hat{v}}) \|_1] \leq 2 \mathbb{E}_r \left[ \frac{\kappa_1(r)}{\kappa_0(r)} \left ( \frac{r}{2^n} \right) \right]$$
\end{corr}

\begin{proof}

This follows from iterated expectation:

\begin{align*}
    \mathbb{E}_{\vec{v} - \vec{\hat{v}} \sim \mathcal{D}}[\| \varphi(\vec{v}) -  \varphi(\vec{\hat{v}}) \|_1 ] & = \mathbb{E}_r \left [\mathbb{E}_{\vec{v} - \vec{\hat{v}} \sim \mathcal{D}_{S_r}} \big [ \| \varphi(\vec{v}) -  \varphi(\vec{\hat{v}}) \|_1  \mid \| \vec{v} - \vec{\hat{v}} \|_1 = r \big ] \right ] \\
    & \leq \mathbb{E}_r \left [ \left ( \frac{\kappa_1(r)}{\kappa_0(r)} \frac{2}{2^n} \right ) r \right ] \\
    & = 2 \mathbb{E}_r \left [ \frac{\kappa_1(r)}{\kappa_0(r)} \frac{r}{2^n} \right]
\end{align*}

where the inequality holds by the Theorem above.

\end{proof}

\textbf{Remark:}  Therefore, if $\mathbb{E}_r [ \frac{\kappa_1(r)}{\kappa_0(r)} r] = c \mathbb{E}_r [r]$ for some constant $c = O(1)$, then the error in the Shapley value is fairly small and proportional to the average expected L1 error $\frac{\mathbb{E}_r[r] }{2^n}$.

\subsection{Discussion about CGA-Specific Errors:}

Since CGA is a \emph{complete representation}, every game may be expressed as a CGA of some order (see Fact 1). And so, we may plug the CGA-specific bias into the general bounds obtained previously in Theorems 3-6.

Below we derive CGA bias due to model misspecification. Note that the approximation is lossy only when the true game is generated by a CGA model of order $r$ and we model it with a simpler CGA model of order $k$ with $k < r$. When we model the game with a CGA of a higher order than the actual game, it is clear that we can learn a set of weights that would fit the coalition values exactly (since $v$ would be in the columnspace).

Let $M^{nk}$ denote the matrix relating the parameters $\vec{\omega}$ to the coalitional values $\vec{v}$. It is a $2^n \times d_k$ matrix of the form:

\begin{center}
\bordermatrix{ ~ & \text{first order weights} & ... & \text{$k$th order weights} \cr
      \text{row corresponding to null coalition $\{\}$} & ... & ... & ... \cr
      ... & ... & ... & ... \cr
      \text{row corresponding to grand coalition $A$} & ... & ... & ... \cr}
\end{center}

The CGA model parameters $\vec{\hat{\omega}}$ we learn will be such that:

$$\vec{\hat{\omega}} = \argmin_{\vec{w}} \| M^{nk} \vec{w} - M^{nr} \vec{\omega_r^*} \|_2^2 $$

This is just equivalent to projecting vector $M^{nr} \vec{\omega_r^*}$ onto the columnspace spanned by $M^{nk}$.
Recall from our identification theorem that $M^{nk}$ has enough rows to be full rank, which makes $(M^{nk})^T M^{nk}$ positive definite and invertible; if there are not enough samples, we may instead consider a regularization term that will make the matrix invertible. Define projection matrix $P_{nk}$:

$$P_{nk} = M^{nk}(((M^{nk})^T M^{nk})^{-1} (M^{nk})^T $$

This means that the misspecification error $e(n, k , r)$ may be expressed as:

$$e(n, k , r) = (I - P_{nk})M^{nr} \vec{\omega_r^*}$$ 

which we may plug into our noise bounds for the Shapley value computation.

Unlike the Shapley matrix, the error matrix $(I - P_{nk})M^{nr}$ does not seem to admit a closed form for its trace.  Instead, we perform simulations to better understand its properties. In particular, we look to understand if it enjoys the same "averaging-effect" as the Shapley matrix. We compute the max eigenvalue and the average trace norm value sweeping over all $n, r, k$ for $k < r < n$ for $n \in [2, 15]$ (we try these sizes since 15 is the largest possible before the error matrix's size exceeds that permitted by our machine memory). Our simulations suggest that its largest eigenvalue (for the worst case bound) and the average trace value (for the average case bound per Lemma \ref{lem: average_singular_value}) both grow monotonically with $n$ and $r$ (fixing a $k$). Altogether, this suggests that the $\ell_2$ error can grow arbitrarily large with model misspecification. 

\subsection{Proofs of Facts}

For completeness, we provide proofs of the two facts listed.

\begin{fact}[Unique decomposition form]
  There exists a unique set of values $\omega_S$ for each subset $S \subseteq A$ with $|S| \leq k$
  such that the characteristic function can be decomposed into its interaction form where
  \begin{align}
  v(C) = \sum_{k=1}^{|C|} \sum_{S \in 2^C_{k}} \omega_{S}.
  \end{align}
\end{fact}

\begin{proof}
We can show this inductively.
For the base case when $|C| = 1$ we have $w_C = v(C)$, which is unique.

Induction step: assume $w_{S'}$ is uniquely determined for $|S'| = 1,..., m - 1$.
Then for a particular subset $|S| = m$:

$$v(S) = \sum_{i=1}^{m-1} \sum_{S' \in 2_{i}^S} w_{S'} + w_S$$

and thus $w_S$ is uniquely determined since we must set it to

$$w(S) = v(S) - \sum_{i=1}^{m-1} \sum_{S' \in 2_{i}^S} w_{S'} $$
\end{proof}

\begin{fact}[Shapley value expression]
The Shapley Value of an individual $i$ with respect to team $A$ can be expressed as: 

$$\varphi_{i}(v) = \sum_{T \subseteq A \setminus \{i\}} \frac{1}{|T|  + 1} \omega_{T \cup \{i\}}$$
\end{fact}

\begin{proof}
The Shapley value for player $i$ is defined as:

$$\varphi_{i}(v)  = \frac{1}{n} \sum_{S \subseteq A \char`\\ \{i\}} \binom{n-1}{|S|}^{-1} (v(S \cup \{i\}) - v(S))$$

Plugging in the decomposition form:

$$v(S \cup \{i\}) - v(S) = \sum_{S' \subseteq S} w_{S' \cup \{i\}}$$

Thus, the Shapley value for $i$ is only a function of all $w_{S}$ where $i \in S$.

Given a subset $T = \{i_1...i_t\}$, let us derive the weighted sum of $w_{T}$  occurrences in $\varphi_{i_1}(v)$.
This term only appears if $\{i_2,..,i_t\} \subseteq S$ but $i_1 \notin S$. 
And so, the weighted sum of occurrences is:

\begin{align*}
    \frac{1}{n} \sum_{S \subseteq A \char`\\ \{i_1\}, \{i_2,..,i_t\} \in S} \binom{n-1}{|S|}^{-1} =  \frac{1}{n} \sum_{s=t-1}^{n-1} \binom{n-1}{s}^{-1} \binom{n - t}{s - (t-1)}
\end{align*}

Similarly, $w_T$ has the same sum of weighted occurrences in expressions for players $i_2,...,i_t$.
And so, by efficiency (since $v(A)$ contains exactly $w_T$ and the sum of Shapley payments equals $v(A)$), they must be assigned equal portions of $w_T$, i.e $w_{T} / |T|$. 
This holds for all subsets T. And so, a player $i$'s Shapley value is the sum of all weights $w_{T} / |T|$, for all subsets $T \subseteq [n]$ and $i \in T$.
\end{proof}

\subsection{Relationship to the Core}
The main text of the paper has focused on the solution concept of the Shapley Value. Another commonly used solution concept in cooperative game theory is known as the Core~\cite{Gillies19593ST}. Let $n$ be the number of players in the game, the Core is an allocation $x \in \mathbb{R}^n$ that satisfies:

(i) Efficiency: $\sum_{i \in [n]} x_i = v([n])$

(ii) Stability: for any coalition $C \subset [n]$: 

$$\sum_{i \in C} x_i \geq v(C)$$ 

Intuitively, a payoff vector is in the Core if it incentivizes every coalition $C$ to stay with the grand coalition rather than leave, achieve a value of $v(C)$ and split it amongst themselves in some other way.

The Core of a game may be empty, though an extension known as the Least Core is always guaranteed to exist. The Least Core can be computed by solving the following linear program:

\begin{equation*}
\begin{array}{cl@{}ll}
\min_{e,x}  & e &\\
 s.t. & \sum_{i \in [n]} x_i = v([n]) &  \\
                 &  \sum_{i \in C} x_i \geq v(C) - e & \quad \forall S \subset [n]
\end{array}
\end{equation*}

Intuitively, the Least Core is the allocation which minimizes the subsidy $e$ required to incentivize all coalitions to stay together. We call the minimum subsidy needed the Least Core value. Unfortunately, ~\cite{deng1994complexity} show that for any CGA model with order higher than $1$, it is NP-Complete to compute the Least Core Value.

One notable allocation in the Least Core is the Nucleolus. For a given allocation $x$, define deficit function $e_x(C) = v(C) - \sum_{i \in C} x_i$. Order all subsets of $[n]$ according to the deficit function $e_x$. The nucleolus is defined as the imputation which lexicographically minimizes this ordering of deficits. Intuitively, the Nucleolus is the "inner-most" allocation in the Least Core. In general, the Nucleolus is difficult to compute and requires solving a series of exponential-size linear programs. 

Remarkably, ~\cite{deng1994complexity} prove the following fact:

\begin{fact}
\label{fact: nucl}
Assuming the characteristic function of the underlying game is a second order CGA model, the Shapley Value is in the Least Core (in fact, it is the Nucleoulus). 
\end{fact}

Therefore, we can simply compute the Shapley value to obtain a point in the Least Core. All that remains is to approximate the Least Core value. To do this, we establish an approximate notion of the Least Core value by adapting a similar notion from \cite{balkanski2017statistical} and derive a simple sample complexity bound for estimating this value. The definition goes as follows:

\begin{defn} 
Given an allocation $\vec{x}$, a value $e$ is a $\delta-$probable least core value if:

$$\Pr_{C \sim 2^A}[\sum_{i \in C} x_i + e \geq v(C)] \geq 1 - \delta$$
\end{defn}

The least core value is the smallest $e^*$ such that there exist an allocation for which $e^*$ is a $0-$probable least core value. 

We will compute a $\delta-$probable least core value by computing the sample least core value on a set of uniformly sampled coalitions $\widehat{S}$. Certainly if $|\widehat{S}| = 2^n$ coalitions, then the sample least core value will be the true least core value exactly. Using standard learning theory tools, we can relate the quality of the estimation of the least core value, in terms of $\delta$, to the size of the samples $\widehat{S}$ needed:

\begin{theorem}

Given a set $\widehat{S}$ of $m = O(\frac{\log(1 / \Delta)}{\delta^2})$ coalitions uniformly sampled from $2^A$, let:

\begin{equation*}
\begin{array}{ll@{}ll}
\hat{e} & = \argmin e &\\
                 &  \qquad \sum_{i \in C} \varphi_i(\hat{v}) \geq v(C) - e & \quad \forall C \subseteq \widehat{S}
\end{array}
\end{equation*}

then with probability $1 - \Delta$ over the samples, $\hat{e}$ is a $\delta-$least core value.
\end{theorem}

\begin{proof}

We prove this through a simple learning theory setup analogous to the proposition above. Define a $2$-dimensional linear classifier with weights $\vec{w_e} = [e, 1]$. This class of classifier is a subset of all linear classifiers of dimension $2$ and thus has VC dimension $\leq 2$. 

For each of the $2^n - 2$ inequality constraints, construct data point $[1, \sum_{i \in C} \varphi_i(\hat{v})  - v(C)]$ that corresponds to coalition $C$'s constraint. We assign each data point a label of $1$. Notice that if classifier $\vec{w_{e}} = [e, 1]$ classifies $[1, \sum_{i \in C}  \varphi_i(\hat{v})  - v(C)]$ correctly, then:

$$\text{sign}_{\vec{w_{e}}}([1, \sum_{i \in C} \varphi_i(\hat{v}) - v(C)]) = 1  \Rightarrow [e, 1]^T [1, \sum_{i \in C} \varphi_i(\hat{v})  - v(C)] \geq 0 \Rightarrow  \sum_{i \in C}  \varphi_i(\hat{v}) \geq v(C) - e$$

Moreover, we know that the classifier we obtain, $\vec{w_{\hat{e}}} = [\hat{e}, 1]$, is such that it classifies all the samples in $\widehat{S}$ correctly by construction, and has zero empirical risk. Again, using Lemma \ref{lem: uniform_convergence}, we know that this classifier's performance on the samples generalize to all $2^n - 2$ constraints. In particular, if there are at least

$$O(\frac{2 +  \log(1 / \Delta)}{\delta^2})$$ 

samples in $\widehat{S}$, then the empirical least core value $\hat{e}$ we compute is such that:

$$\Pr_{C \sim 2^A} [ \sum_{i \in C}  \varphi_i(\hat{v}) \geq v(C) - \hat{e} ] = \Pr_{C \sim 2^A} [\text{sign}_{\vec{w_{\hat{e}}}}([1, \sum_{i \in C}  \varphi_i(\hat{v})  - v(C)]) = 1 ] \geq 1 - \delta $$

\end{proof}

Lastly, we remark that for games whose characteristic functions are CGA models of order higher than $2$, the Shapley is not the Nucleolus. An interesting extension of this work could be developing faster, sample-based methods for computing the Least Core with higher order CGA models.

\newpage

\subsection{Experiments Hyper Parameter Search}
In the low rank approximations of $\widehat{V}$ (as suggested by \cite{li2018learning,seshadri2019discovering}), we represented a team $C$ via a one-hot encoding $\vec{x_C}$ and fit a model of the form:
$$\hat{v}(C) = \vec{w}^T \vec{x_C} + \vec{x_C}^T \hat{V} \vec{x_C}$$ 

We tried parameterizing $\hat{V}$ via a low-rank matrix and swept weight decay ($l_2$ regularization) parameters on our validation set. Here we report the results of the full sweep for both of our experiments.

Table \ref{rl_experiment} shows the MSE (lower is better) of the performance prediction for various parameter values in the OpenAI particle world experiment. Table \ref{NBA_experiment} shows the accuracy of the model in predicting wins (higher is better) in the NBA experiment. In both cases we see that a relatively low rank model does very well at capturing structure in our environments. The main text analyzes the models resulting from these parameter choices.

\begin{table}
\[\begin{array}{c|cccccc}
\text{L2 regularization}/\text{$\hat{V}$ Rank} & 1 & 2 & 5 & 10 & 20 & 35\\
\hline
0.001 & 0.256 & 0.092 & \textbf{0.066} & 0.067 & 0.068 & 0.069\\
0.01 & 0.261 & 0.104 & 0.091 & 0.090 & 0.093 & 0.090\\
0.1 & 0.679 & 0.679 & 0.652 & 0.646 & 0.669 & 0.664\\
\end{array}\]
\caption{Results of hyper-parameter sweep for the second order CGA in the OpenAI particle world experiment. MSE is shown, lower is better.}\label{rl_experiment}
\end{table} 

\begin{table}\label{NBA_experiment}
\[\begin{array}{c|ccccc}
\text{L2 regularization}/\text{$\hat{V}$ rank} & 5 & 10 & 20 & 50 & 100\\
\hline
0.001 & 0.61 & 0.6214 & 0.6086 & 0.5971 & 0.5929 \\
0.01 & 0.64 & 0.6414 & 0.6414 & \mathbf{0.6429} & 0.6429\\
0.1 & 0.6214 & 0.63 & 0.62 & 0.6286 & 0.6242\\
\end{array}\]
\caption{Results of hyper-parameter sweep for the second order CGA in the NBA data. Model accuracy is shown, higher is better.}
\end{table}

\newpage

\end{document}